\documentclass{article}

\PassOptionsToPackage{numbers, sort&compress}{natbib}
\usepackage[final]{neurips_2022}

\usepackage[utf8]{inputenc} %
\usepackage[T1]{fontenc}    %
\usepackage{hyperref}       %
\usepackage{url}            %
\usepackage{booktabs}       %
\usepackage{amsfonts}       %
\usepackage{nicefrac}       %
\usepackage{microtype}      %
\usepackage{xcolor}         %

\usepackage{amsmath,amsfonts,bm}

\def\eqref#1{(\ref{#1})}

\def\1{\bm{1}}

\def\rd{{\textnormal{d}}}

\def\vv{{\bm{v}}}

\def\mI{{\bm{I}}}

\def\mX{{\bm{X}}}

\DeclareMathAlphabet{\mathsfit}{\encodingdefault}{\sfdefault}{m}{sl}
\SetMathAlphabet{\mathsfit}{bold}{\encodingdefault}{\sfdefault}{bx}{n}

\def\sR{{\mathbb{R}}}

\newcommand{\E}{\mathbb{E}}

\newcommand{\KL}{D_{\mathrm{KL}}}

\DeclareMathOperator*{\argmin}{arg\,min}

\DeclareMathOperator{\sign}{sign}
\DeclareMathOperator{\Tr}{Tr}

\definecolor{Purple200}{HTML}{E040FB}
\definecolor{Purple400}{HTML}{D500F9}
\definecolor{DeepPurpleA400}{HTML}{651FFF}
\definecolor{Indigo400}{HTML}{3D5AFE}
\definecolor{Green400}{HTML}{00E676}
\definecolor{Green700}{HTML}{00C853}
\definecolor{Amber800}{HTML}{FF8F00}
\definecolor{Orange800}{HTML}{EF6C00}
\definecolor{DeepOrange800}{HTML}{D84315}
\definecolor{DeepOrangeA400}{HTML}{FF3D00}
\definecolor{RedA400}{HTML}{FF1744}

\newcommand*\target{ \rho_\text{target} }%
\newcommand*\Zhat{ \widehat{Z} }%
\newcommand*\Yhat{ \widehat{Y} }%
\newcommand*\Psihat{ \widehat{\Psi} }%
\newcommand*\Xbar{ \bar{X} }%
\newcommand*\Lipf{ \calL_\text{IPF} }%
\newcommand*\Ltd{ \calL_\text{TD} }%
\newcommand*\Lfk{ \calL_\text{FK} }%

\newcommand*\TD{ {\mathrm{TD}} }%
\newcommand*\TDhat{ \widehat{\mathrm{TD}} }%

\newcommand{\norm}[1]{\lVert#1\rVert}

\def\dt{{ \mathrm{d} t }}
\def\ds{{ \mathrm{d} s }}

\def\calB{{\cal B}}

\def\calF{{\cal F}}

\def\calL{{\cal L}}

\def\calN{{\cal N}}
\def\calO{{\cal O}}
\def\calP{{\cal P}}

\def\calW{{\cal W}}

\newcommand{\fracpartial}[2]{\frac{\partial #1}{\partial  #2}}

\newcommand{\br}[1]{\left[#1\right]}
\newcommand{\pr}[1]{\left(#1\right)}
\newcommand{\T}{\top}
\newcommand*\bvec[1]{\begin{bmatrix}#1\end{bmatrix}}

\newcommand{\eg}{{\ignorespaces\emph{e.g.,}}{ }}
\newcommand{\ie}{{\ignorespaces\emph{i.e.,}}{ }}

\usepackage{mathtools}
\usepackage{tabularx}
\usepackage{framed}
\usepackage{bbm}

\usepackage{amssymb}
\usepackage{amsthm}
\usepackage{multirow}

\newtheorem{theorem}{Theorem}
\newtheorem{lemma}[theorem]{Lemma}
\newtheorem{proposition}[theorem]{Proposition}

\newcommand\numberthis{\addtocounter{equation}{1}\tag{\theequation}}
\usepackage{cancel}
\usepackage{lipsum}

\usepackage{capt-of}
\usepackage{wrapfig}

\usepackage{xcolor}
\usepackage{color,soul}
\colorlet{color1}{green!50!black}
\colorlet{color2}{orange!95!black}
\colorlet{color3}{red!80!black}
\colorlet{color4}{red!65!black}
\colorlet{color5}{blue!75!green}
\colorlet{blueee}{blue!50!black}

\definecolor{label1}{HTML}{99292A}
\definecolor{label2}{HTML}{D89A3C}
\definecolor{label3}{HTML}{417481}
\colorlet{label22}{label2!80!black}

\definecolor{amaranth}{rgb}{0.9, 0.17, 0.31}

\newcommand{\markgreen}[1]{{\color{color1} #1}}

\newcommand{\markblue}[1]{{\ignorespaces\color{color5} #1}}

\newcommand{\markreddd}[1]{\ignorespaces{\color{color4} #1}}

\newcommand{\markaa}[1]{\ignorespaces{\color{label1} #1}}
\newcommand{\markbb}[1]{\ignorespaces{\color{label22} #1}}
\newcommand{\markcc}[1]{\ignorespaces{\color{label3} #1}}

\usepackage{footnote}

\usepackage{pifont}
\usepackage{ifsym}
\newcommand{\cmark}{{\ding{51}}}%
\newcommand{\xmark}{{\ding{55}}}%

\let\oldsqrt\sqrt
\def\sqrt{\mathpalette\DHLhksqrt}
\def\DHLhksqrt#1#2{\setbox0=\hbox{$#1\oldsqrt{#2\,}$}\dimen0=\ht0
\advance\dimen0-0.2\ht0
\setbox2=\hbox{\vrule height\ht0 depth -\dimen0}%
{\box0\lower0.4pt\box2}}

\newcommand{\specialcell}[2][c]{%
  \begin{tabular}[#1]{@{}c@{}}#2\end{tabular}}

\usepackage{enumitem}
\usepackage[position=top]{subfig}
\usepackage{arydshln}

\usepackage{algorithm}
\usepackage{algorithmic}

\usepackage{booktabs}       %

\usepackage{enumitem}
\usepackage{tikz}

\usepackage{empheq}
\newcommand*\widefbox[1]{\fbox{\hspace{1em}#1\hspace{1em}}}
\usepackage{tablefootnote}

\usepackage{textcomp}
\usepackage{colortbl}

\setul{2pt}{.4pt}

\makeatletter
\newtheorem*{rep@theorem}{\rep@title}
\newcommand{\newreptheorem}[2]{%
\newenvironment{rep#1}[1]{%
 \def\rep@title{#2 \ref{##1}}%
 \begin{rep@theorem}}%
 {\end{rep@theorem}}}
\makeatother

\newreptheorem{theorem}{Theorem}
\newreptheorem{lemma}{Lemma}
\newreptheorem{proposition}{Proposition}

\usepackage{aligned-overset}
\usepackage{cases}

\newcommand*\SB{Schr{\"o}dinger Bridge }%
\newcommand*\DeepGSB{Deep Generalized Schr{\"o}dinger Bridge }%

\hypersetup{
    colorlinks,
    linkcolor={blue!50!black},
    citecolor={blue!50!black},
    urlcolor={blue!80!black}
}

\title{\DeepGSB}

\author{%
  Guan-Horng Liu$^1$, Tianrong Chen$^{1}$\thanks{
    These authors contributed equally. Work was done while Oswin was at Georgia Tech.}
    , Oswin So$^{2*}$, Evangelos A. Theodorou$^1$\\
  $^1$Georgia Institute of Technology, USA\\
  $^2$Massachusetts Institute of Technology, USA\\
  \texttt{\{ghliu, tianrong.chen, evangelos.theodorou\}@gatech.edu}\\
  \texttt{oswinso@mit.edu}
}

\begin{document}

\maketitle

\begin{abstract}

Mean-Field Game (MFG) serves as a crucial mathematical framework
in modeling the collective behavior of individual agents interacting stochastically with a large population.
In this work,
we aim at solving a challenging class of MFGs %
in which the differentiability of these interacting preferences may \textit{not} be available to the solver,
and the population is urged to \textit{converge exactly} to some desired distribution.
These setups are, despite being well-motivated for practical purposes, complicated enough to paralyze most (deep) numerical solvers.
Nevertheless, we show that \SB
--- as an entropy-regularized optimal transport model ---
can be generalized to accepting mean-field structures, hence solving these MFGs.
This is achieved
via the application of Forward-Backward Stochastic Differential Equations theory,
which, intriguingly, leads to a computational framework
with a similar structure to Temporal Difference learning.
As such, it opens up novel algorithmic connections to Deep Reinforcement Learning
that we leverage to facilitate practical training.
We show that our proposed objective function
provides necessary and sufficient conditions to the mean-field problem.
Our method, named \DeepGSB (\textbf{DeepGSB}),
not only outperforms prior methods in solving classical population navigation MFGs,
but is also capable of solving
{1000}-dimensional \textit{opinion depolarization}, setting a new state-of-the-art numerical solver for
high-dimensional MFGs.
Our code will be made available at \url{https://github.com/ghliu/DeepGSB}.

\end{abstract}

\setcounter{footnote}{0}

\section{Introduction} \label{sec:1}

\begin{wrapfigure}[14]{r}{0.27\textwidth}
    \vspace{-20pt}
    \begin{center}
      \includegraphics[width=0.26\textwidth]{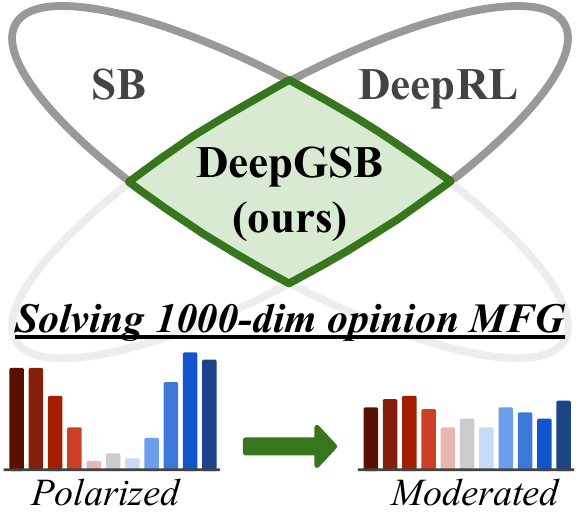}
    \end{center}
    \vskip -0.05in
    \caption{
        \textbf{DeepGSB} paves a new algorithmic connection between \SB (SB) and model-based DeepRL
        for solving high-dimensional MFGs.
    }
    \label{fig:1}
\end{wrapfigure}
On a scorching morning, you {navigated through the crowds} toward the office.
As you walked through a crosswalk, you were pondering {the growing public opinion} on a new policy over the past week,
and were suddenly interrupted by the honking as the traffic started moving...

From navigation in crowds to propagation of opinions {and traffic movement},
examples of \textit{individual agents interacting with a large population}
are widespread in daily life
and, due to their prevalence, appear as an important subject
in multidisciplinary scientific areas, including %
economics \citep{achdou2014partial,achdou2022income},
opinion modeling \citep{schweighofer2020agent,gaitonde2021polarization,hkazla2019geometric},
robotics \citep{liu2018mean,elamvazhuthi2019mean},
and more recently machine learning \citep{lu2020mean,hu2019mean,han2018mean}.

Mathematically, the decision-making processes under these scenarios
can be characterized by the \textbf{Mean-Field Game} \citep{lasry2007mean,gueant2011mean,bensoussan2013mean} \textbf{(MFG)},
which models a noncooperative differential game on a finite horizon %
between a continuum population of rational agents.
Let $u(x,t)$ be
the \textit{value} function, also known as optimal cost-to-goal, that governs agents' policies at each state $x\in \sR^d$ and time $t\in[0,T]$,
and denote the resulting population density by $\rho(\cdot,t) \in \calP(\sR^d)$,
where $\calP(\sR^d)$ is the set of probability measures on $\sR^d$.
At the Nash equilibrium where no agent has the incentive to change his/her decision,
MFG, at its
most general form, solves the following partial differential equations (PDEs):
\begin{figure}[t]
  \vskip -0.1in
  \centering
  \captionsetup{type=table}
  \caption{
    Comparison to existing methods w.r.t. various desired features in Mean-Field Games (MFGs).
    Our \textbf{DeepGSB} is capable of solving a much wider class of MFGs in higher dimensional state spaces.
  }
  \vskip -0.05in
  \centering
  \begin{tabular}{lccccc}
    \toprule
    & \specialcell[c]{continuous \\ state space}
    & \specialcell[c]{stochastic \\ MF dyn. \eqref{eq:mfg-sde}}
    & \specialcell[c]{converges to \\ exact $\target$}
    & \specialcell[c]{discontinuous \\ MF interaction $F$}
    & \specialcell[c]{highest \\ dimension }
    \\
    \midrule
    \citet{ruthotto2020machine} & \cmark & \xmark & \xmark & \xmark & 100 \\
    \citet{lin2021alternating} & \cmark & \xmark & \xmark & \xmark & 100 \\ %
    \citet{chen2021density} & \xmark & \cmark & \cmark & \xmark\protect\footnotemark & 2 \\
    \midrule
    \textbf{DeepGSB (ours)} & \cmark & \cmark & \cmark & \cmark & {\textbf{1000}} \\
    \bottomrule
  \end{tabular} \label{table:1}
  \vskip -0.12in
\end{figure}
\footnotetext{
  Precisely, \citet{chen2021density} considered discontinuous yet {non-MF} interaction, $F{:=}F(x)$, on a \textit{discrete} state space.
}
\begin{equation}
    \begin{split}
        \left\{
        \begin{array}{lr@{=}l}
        - \fracpartial{u(x,t)}{t} + H(x, \nabla u, \rho) - \frac{1}{2}\sigma^2\Delta u = F(x, \rho),
        &\quad u(x,T)~&~G(x, \rho(\cdot,T)) \\[3pt]
        \fracpartial{\rho(x,t)}{t} - \nabla \cdot (\rho~\nabla_p H(x,\nabla u, \rho) - \frac{1}{2}\sigma^2\Delta \rho = 0,
        &\quad \rho(x,0)~&~\rho_0(x)
        \end{array}
        \right.,
    \label{eq:mfg-pde}
    \end{split}
\end{equation}
where $\nabla$, $\nabla \cdot$, and $\Delta$ are respectively the gradient, divergence, and Laplacian operators.\footnote{
    These operators are taken w.r.t. $x$ unless otherwise noted. See Appendix~\ref{sec:a1} for the notational summary.
}
These two PDEs are respectively known as the Hamilton-Jacobi-Bellman (HJB) and Fokker-Plank (FP) equations,
which characterize the evolution of $u(x,t)$ and $\rho(x,t)$. %
They are coupled with each other through
the Hamiltonian $H(x,p, \rho): \sR^d \times \sR^d \times \calP(\sR^d) \rightarrow \sR$, which describes the dynamics of the game, and
the mean-field interaction $F(x,\rho): \sR^d \times \calP(\sR^d) \rightarrow \sR$, which quantifies the agent's preference when interacting with the population.
The terminal condition $G$ typically penalizes deviations from some desired target distribution $\target$, \eg $G \approx \KL(\rho(\cdot,T)||\target(\cdot))$.
Given a solution $(u,\rho)$ to \eqref{eq:mfg-pde}, each agent acts accordingly and follows a stochastic differential equation(SDE)
\begin{equation}
    \rd X_t = - \nabla_p H(X_t,\nabla u(X_t,t), \rho(\cdot, t)) \dt + \sigma \rd W_t,
    \quad X_0 \sim \rho_0, %
    \label{eq:mfg-sde}
\end{equation}
where $W_t \in \sR^d$ is the Wiener process and $\sigma\in\sR$ is some diffusion scalar.
At the mean-field limit, \ie when the number of agents goes to infinity,
the collective behavior of \eqref{eq:mfg-sde} yields the density $\rho(\cdot,t)$.

Numerical methods for solving \eqref{eq:mfg-pde} have advanced rapidly with the aid of machine learning.
Seminar works such as \citet{ruthotto2020machine} and \citet{lin2021alternating} approximated
$(u,\rho)$ with deep neural networks (DNNs) and directly penalized the violation of PDEs.
Despite showing preliminary successes,
the underlying dynamics \eqref{eq:mfg-sde} were either degenerate (\eg $\sigma:=0$) \citep{ruthotto2020machine},
or completely discarded by instead regressing network outputs on the entire state space \citep{lin2021alternating},
which can scale unfavorably as the dimension $d$ grows. %
An alternative that avoids both limitations,
\ie it keeps the full stochastic dynamics in \eqref{eq:mfg-sde} while being computationally scalable,
is to recast these PDEs to a set of forward-backward SDEs (FBSDEs) by applying the
nonlinear Feynman-Kac Lemma \citep{han2018solving,exarchos2018stochastic,pereira2019neural}.
The FBSDEs analysis appears extensively in the theoretical study of MFG
\citep{carmona2013control,carmona2013mean,carmona2019convergence,carmona2021convergence},
yet development of scalable FBSDEs-based solver has remained, surprisingly, limited.
Our work contributes to this direction.

Since $\target$ is known in prior, in many cases
there are direct interests
to seek an optimal policy that guides the agents from an initial distribution $\rho_0$ to the \textit{exact} $\target$, while respecting the structure of MFG,
particularly the MF interaction $F(x,\rho)$.
Lifting \eqref{eq:mfg-pde} to this setup, however, is highly nontrivial.
Indeed, replacing the \textit{soft} penalty at $u(x,T) = \KL(\rho || \target)$ with a \textit{hard} distributional constraint at $\rho(x,T) = \target$ yields an HJB whose boundary condition can only be defined implicitly through FP, which now contains two distributional constraints and resembles an optimal transport problem.
As such, despite being well-motivated, {most prior methods have struggled to extend to this setup}.

In this work, we show that \textbf{\SB}(\textbf{SB}),
as an entropy-regularized optimal transport problem \citep{de2021diffusion,vargas2021solving,wang2021deep,chen2021likelihood,bunne2022recovering},
provides an elegant recipe for solving this challenging class of \textit{MFGs with distributional boundary constraints ($\rho_0, \target$)}.
Although SB is traditionally set up with $F := 0$ \citep{schrodinger1932theorie,caluya2021wasserstein,backhoff2020mean},
we show that SB-FBSDE \citep{chen2021likelihood}, an FBSDE-based method for solving SB,
can be generalized to accept nontrivial $F$; hence solving MFG.
Interestingly,
the new FBSDEs system
admits a similar computational structure to temporal difference (TD) learning,
leading to a framework that narrows the gap between SB and Deep Reinforcement Learning (DeepRL); see Fig.~\ref{fig:1}.
This connection enables our method to take advantage of DeepRL techniques,
such as target networks, replay buffer,
actor-critic, \textit{etc},
and, more importantly,
to handle a wide class of MF interactions
that need \textit{not} be continuous \textit{nor} differentiable.
This is in contrast to most existing works,
which require differentiable \citep{ruthotto2020machine,lin2021alternating} or quadratic \citep{chen2015optimal} structure on $F$, or discretize the state space \citep{chen2021density}.
We validate our method, called \textbf{Deep Generalized \SB (DeepGSB)}, on various challenging MFGs
from crowd navigation to \textit{high-dimensional opinion depolarization} (where $d=${1000}),
setting a state-of-the-art record in the area of numerical MFG solvers.

In summary, we present the following contributions.
\begin{itemize}[leftmargin=13pt]
    \item
    We present a novel numerical method, rooted in Schr{\"o}dinger Bridge (SB), for solving a challenging class of
    Mean-Field Game
    where the population needs to converge \textit{exactly} to the target distribution.

    \item
    The resulting method, \textbf{DeepGSB}, generalizes prior SB results to accepting
    flexible mean-field interaction (\eg non-differentiable)
    and enjoys modern training techniques from DeepRL.

    \item
    \textbf{DeepGSB} achieves promising empirical results in navigating crowd motion and
    depolarizing 1000-dimensional opinion dynamics,
    setting a new state-of-the-art numerical MFG solver.

\end{itemize}

\section{Preliminary on \SB (SB)} \label{sec:2}

\begin{wrapfigure}[12]{r}{0.21\textwidth}
    \vspace{-40pt}
    \begin{center}
      \includegraphics[height=0.18\textwidth, width=0.2\textwidth]{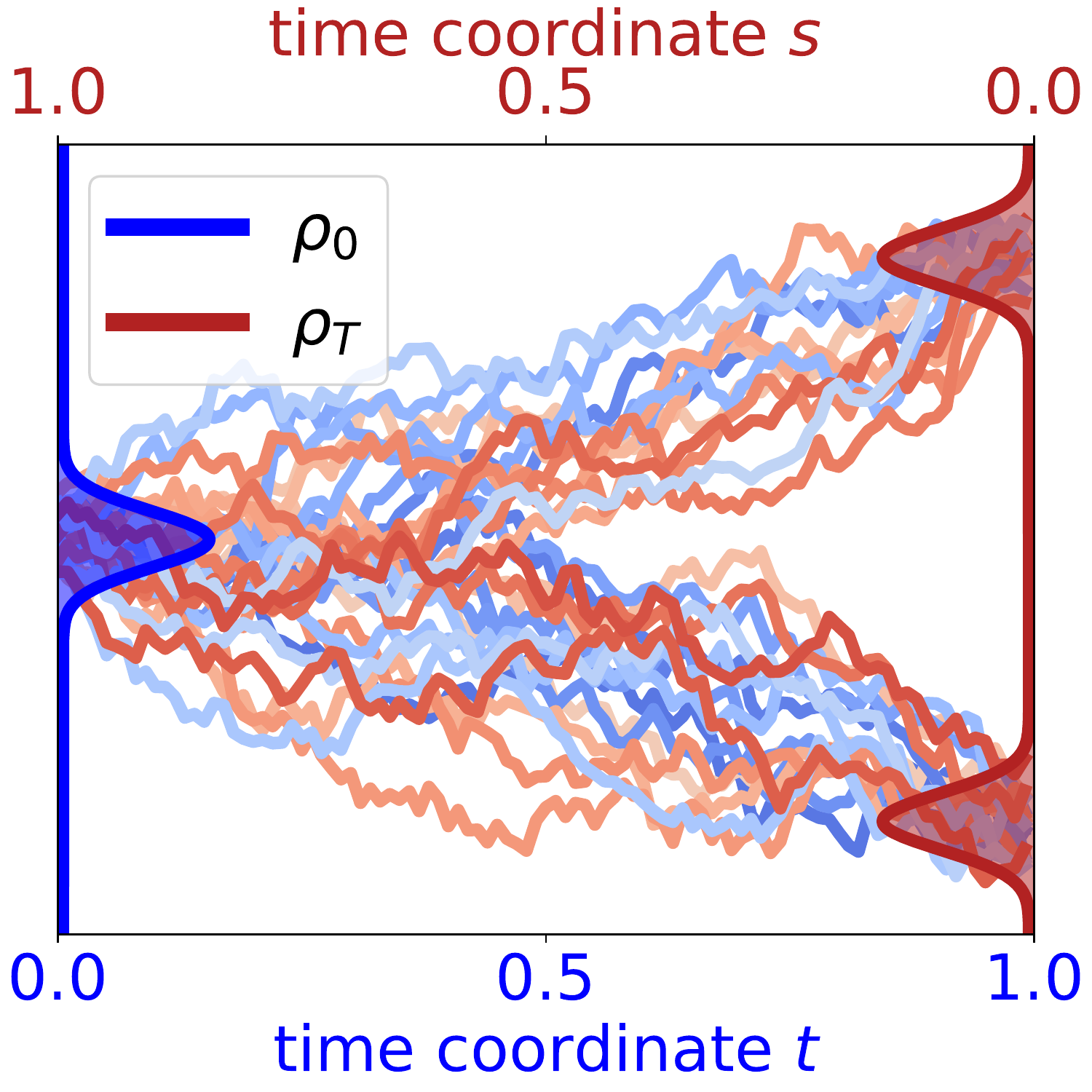}
    \end{center}
    \vskip -0.1in
    \caption{
        Simulation of the \markblue{forward} \eqref{eq:sde} and \markreddd{backward} \eqref{eq:rsde} SDEs in SB, which are minimum-energy solution when $(\Psi,\Psihat)$ obey the PDEs in \eqref{eq:sb-pde}.
    }
    \label{fig:2}
\end{wrapfigure}
The SB problem was originally introduced in the 1930s for quantum mechanics \citep{schrodinger1931umkehrung,schrodinger1932theorie} and later draws broader interests with its connection to optimal transport and control \citep{leonard2012schrodinger,leonard2013survey,pavon1991free,dai1991stochastic}.
Given a pair of boundary distributions $(\rho_0, \rho_T)$, SB seeks
an optimal pair of stochastic processes of the forms: %
\begin{subequations}
    \label{eq:sb-sde}
    \begin{align}
        \rd X_t &= [f(X_t, t) + \sigma^2~\nabla \log {\Psi}(X_t, t) ] \dt + \sigma~\rd W_t, &&X_0 \sim \rho_0, \label{eq:sde}
        \\
        \rd \Xbar_s &= [-f(\Xbar_s, s) + \sigma^2~\nabla \log \Psihat(\Xbar_s, s) ] \ds + \sigma~\rd W_s,  &&\Xbar_0 \sim \rho_T. \label{eq:rsde}
    \end{align}
\end{subequations}
While $X_t$ is a standard stochastic process starting from $\rho_0$,
$\Xbar_s$ evolves along the ``\textit{reversed}'' time coordinate $s:=T-t$ from $\rho_T$.
The base drift $f$ and diffusion $\sigma$ are typically known in prior and related to the Hamiltonian $H$.
Suppose $\Psi, \Psihat \in C^{2,1}(\sR^d,[0,T])$
solve the following coupled PDEs,
\begin{align}
    \begin{cases}
    \fracpartial{\Psi(x,t)}{t}    = - \nabla \Psi^\T f - \frac{1}{2} \sigma^2 \Delta \Psi \\[3pt]
    \fracpartial{\Psihat(x,t)}{t} = - \nabla \cdot (\Psihat f) + \frac{1}{2} \sigma^2 \Delta \Psihat
    \end{cases}
    \text{s.t. }
    \begin{array}{{r@{=}l}}
        \Psi(\cdot,0) \Psihat(\cdot,0)~&~\rho_0 \\[3pt]
        \Psi(\cdot,T) \Psihat(\cdot,T)~&~\rho_T
    \end{array},  \label{eq:sb-pde}
\end{align}
then the theory of SB suggests that
the SDEs in \eqref{eq:sb-sde} are optimal solution to an entropy-regularized (\ie minimum control)
optimization problem.
Furthermore, the path-wise measure induced by \eqref{eq:sde} along $t\in[0,T]$ is equal almost surely to the path-wise measure induced by \eqref{eq:rsde} along $s:=T-t$. In other words, the two SDEs in \eqref{eq:sb-sde} can be thought of as the ``\textit{reversed}'' process to each other; and hence we also have $X_T \sim \rho_T$ and $\Xbar_T \sim \rho_0$ (see Fig.~\ref{fig:2}).

Due to the coupling constraints at the boundaries, solving \eqref{eq:sb-pde} is no easier than solving \eqref{eq:mfg-pde}.
Fortunately, recent advances \citep{chen2021likelihood,bunne2022recovering} have demonstrated a computationally scalable numerical method via the application of the nonlinear Feynman-Kac (FK) Lemma
--- a mathematical tool that recasts certain classes of PDEs into sets of
forward-backward SDEs (FBSDEs) via some transformation.
These \textit{nonlinear FK transformations} are parametrized in SB-FBSDE \citep{chen2021likelihood}
by some DNNs with $\theta$ and $\phi$, \ie
\begin{align}
    Z_\theta(\cdot, \cdot) \approx \sigma~\nabla \log \Psi(\cdot, \cdot)
    \quad \text{ and } \quad
    \Zhat_\phi(\cdot, \cdot) \approx \sigma~\nabla \log \Psihat(\cdot, \cdot),
    \label{eq:nkc-z}
\end{align}
and the FBSDEs resulting from \eqref{eq:sb-pde} and \eqref{eq:nkc-z} yield the following objectives (see Appendix~\ref{sec:a2}):
\begin{subequations}
    \label{eq:L-ipf}
    \begin{align}
        \Lipf(\theta) &= \int_0^T \E\br{
            \frac{1}{2} \norm{Z_\theta(\Xbar_s, s)}_2^2 +  Z_\theta(\Xbar_s, s)^\T \Zhat_\phi(\Xbar_s, s) + \nabla\cdot (\sigma Z_\theta(\Xbar_s, s) {+} f)
        } \ds, \label{eq:L-ipf1} \\
        \Lipf(\phi) &= \int_0^T \E\br{
            \frac{1}{2} \norm{\Zhat_\phi(X_t, t)}_2^2 + \Zhat_\phi(X_t, t)^\T Z_\theta(X_t,t) + \nabla\cdot(\sigma\Zhat_\phi(X_t, t) {-} f)
        } \dt. \label{eq:L-ipf2}
    \end{align}
\end{subequations}
The following lemma, as a direct consequence of \citet{vargas2021machine},
suggests that these objectives can be interpreted as the KL divergences between the parametrized path measures.
\begin{lemma} \label{lemma:kl-sbfbsde}
    Let $q^\theta$ and $q^\phi$ be the path-wise densities of the parametrized forward and backward SDEs
        \begin{align*}
        \rd X^\theta_t = \pr{f(X^\theta_t,t) + \sigma Z_\theta(X^\theta_t,t) } \dt + \sigma \rd W_t, \quad
        \rd \Xbar^\phi_s = \pr{-f(\Xbar^\phi_s,t) + \sigma \Zhat_\phi(\Xbar^\phi_s,t) } \ds + \sigma \rd W_s.
        \end{align*}
    Then, we have
    \begin{align*}
        \KL(q^\theta || q^\phi) \propto \Lipf(\phi), \quad \text{ and } \quad
        \KL(q^\phi || q^\theta) \propto \Lipf(\theta).
    \end{align*}
\end{lemma}
\begin{proof}
    See Appendix~\ref{sec:a3.lemma}.
\end{proof}
Lemma \ref{lemma:kl-sbfbsde} suggests that alternative minimization between $\Lipf(\phi)$ and $\Lipf(\theta)$ is equivalent to performing iterative KL projection \citep{benamou2015iterative}, and is hence equivalent to applying the Iterative Proportional Fitting \citep{kullback1968probability} (IPF) algorithm to solve parametrized SBs \citep{de2021diffusion,vargas2021solving}.

\section{Deep Generalized \SB (DeepGSB)} \label{sec:3}

\subsection{Connection between the coupled PDEs in MFG and SB} \label{sec:3.1}

\begin{wrapfigure}[7]{r}{0.47\textwidth}
    \vspace{-22pt}
    \begin{center}
        \includegraphics[height=2.2cm]{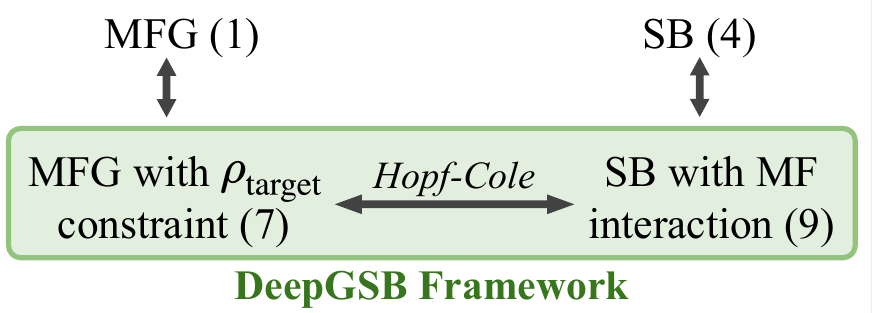}
    \end{center}
    \vskip -0.1in
    \caption{
        Connection between different coupled PDEs appearing in MFG, SB, and DeepGSB.
    }
    \label{fig:3}
\end{wrapfigure}
We begin by first stating our problem of interest
--- MFG with hard distributional constraints $(\rho_0, \target)$ --- in its mathematical form.
Similar to prior works \citep{ruthotto2020machine,lin2021alternating}, we will adopt the control-affine Hamiltonian, %
$H(x,\nabla u,\rho) := \frac{1}{2}\norm{\sigma\nabla u}^2 - \nabla u^\T f(x,\rho)$,
given some base drift $f$ and diffusion scalar $\sigma$.
Substituting this control-affine Hamiltonian into the PDEs in \eqref{eq:mfg-pde} yields
\begin{equation}
    \begin{split}
    \begin{cases}
    - \fracpartial{u(x,t)}{t} + \frac{1}{2}\norm{\sigma\nabla u}^2 - \nabla u^\T f - \frac{1}{2}\sigma^2\Delta u = F(x, \rho), \\[3pt]
    \fracpartial{\rho(x,t)}{t} - \nabla \cdot (\rho~(\sigma^2\nabla u-f)) - \frac{1}{2}\sigma^2\Delta \rho = 0,
    \quad \rho(x,0) = \rho_0(x),\text{ } \rho(x,T) = \target(x),
    \end{cases}
    \label{eq:mfg2-pde}
    \end{split}
\end{equation}
which, as we briefly discussed in Sec.\ref{sec:1}, differ from \eqref{eq:mfg-pde}
in that the boundary condition of the HJB, $u(x,T)$, is now absorbed into FP and defined implicitly through
$\rho(x,T) = \target(x)$.
Since analytic conversion between the boundary conditions of \eqref{eq:mfg-pde} and \eqref{eq:mfg2-pde} exists only for highly degenerate\footnote{
    \citet{zhang2021path} suggested $F:=0$, $f:= f(x)$ and $\rho_0$ a degenerate Dirac delta distribution.
} cases \citep{zhang2021path}, this seemingly innocuous change suffices to paralyze most prior methods.\footnote{
    For completeness, we note that when the base drift is independent of the density, $f:= f(x)$, and mean-field preference,
    $\calF(\rho): F(x,\rho) = \frac{\delta \calF}{\delta \rho}$, is convex in $\rho$, the variational optimization inherited in \eqref{eq:mfg2-pde} remains convex.
    In these cases, the discretized problems converge to the global solution \citep{chen2015optimal,chen2021density}.
    However, for generic mean-field dynamics, such as the polarized $f(x,\rho)$ in our \eqref{eq:opinion-f}, the problem is in general non-convex; hence only local convergence can be established (see \eg Remark 1 in \citep{chen2021density}).
}
Nevertheless,
as \eqref{eq:mfg2-pde} now describes a transformation between two distributions (from FP) while obeying some optimality (from HJB),
it suggests a deeper connection to optimal transport, and hence the SB.

To bridge these new MFG PDEs \eqref{eq:mfg2-pde} to the PDEs appearing in SB \eqref{eq:sb-pde},
we follow standard treatment \citep{caluya2021wasserstein}
and apply the Hopf-Cole transform \citep{hopf1950partial,cole1951quasi}: %
    \begin{align}
    \Psi(x,t) := \exp\pr{-u(x,t)}, \quad
     \Psihat(x,t) := \rho(x,t)\exp\pr{u(x,t)},
    \label{eq:hf}
    \end{align}
which, after some algebra (see Appendix~\ref{sec:a4.1} for details), yields the following PDEs:
\begin{align}
    \left\{
    \begin{array}{l}
    \fracpartial{\Psi(x,t)}{t}
        = - \nabla \Psi^\T f - \frac{1}{2} \sigma^2 \Delta \Psi~
            \markblue{ +~F \Psi} \\[3pt]
    \fracpartial{\Psihat(x,t)}{t}
        = - \nabla \cdot (\Psihat f) + \frac{1}{2} \sigma^2 \Delta \Psihat~
            \markblue{ -~F \Psihat}
    \end{array}
    \right.\text{s.t.}
    \begin{array}{l}
        \Psi(\cdot,0) \Psihat(\cdot,0) = \rho_0 \\[3pt]
        \Psi(\cdot,T) \Psihat(\cdot,T) = \target
    \end{array}.  \label{eq:sb2-pde}
\end{align}
It can be seen that \eqref{eq:sb2-pde} generalizes \eqref{eq:sb-pde} by introducing \markblue{the MF interaction $F$}.
Let $(\Psi, \Psihat)$ be the solution to these new MF-extended PDEs in \eqref{eq:sb2-pde},
and recall the Hamiltonian adopted in \eqref{eq:mfg2-pde},
one can find that
\begin{align*}
    - \nabla_p H(X_t,\nabla u, \rho)
    = f - \sigma^2 \nabla u
    = f + \sigma^2 \nabla \log \Psi.
\end{align*}
That is,
the agent's dynamic \eqref{eq:mfg-sde} coincides with the forward SDE \eqref{eq:sde} in SB.
Hence, we have connected the MFG \eqref{eq:mfg-pde} and SB \eqref{eq:sb-pde} frameworks through the PDEs in \eqref{eq:mfg2-pde} and \eqref{eq:sb2-pde};
see Fig.~\ref{fig:3}.

\subsection{Generalized SB-FBSDEs with mean-field interaction} \label{sec:3.2}

With \eqref{eq:sb2-pde}, we are ready to present our result that {generalizes} prior FBSDE for SB to MF interaction.

\begin{theorem}[Generalized SB-FBSDEs] \label{thm:mfg-fbsde}
    Suppose $\Psi, \Psihat \in C^{2,1}$ and let $f,F$ satisfy {usual growth and Lipchitz conditions} {\normalfont\citep{yong1999stochastic,kobylanski2000backward}}.
    Consider the following nonlinear FK transformations applied to \eqref{eq:sb2-pde}:
    \begin{equation}
    \begin{alignedat}{2}
      Y_t \equiv Y(X_t, t) &= \log \Psi(X_t, t), \qquad
      Z_t \equiv Z(X_t, t) &&= \sigma~\nabla \log \Psi(X_t, t), \\
      \Yhat_t \equiv \Yhat(X_t, t) &=          \log \Psihat(X_t, t), \qquad
      \Zhat_t \equiv \Zhat(X_t, t) &&= \sigma~\nabla \log \Psihat(X_t, t),
    \end{alignedat} \label{eq:nkc-yz}
    \end{equation}
    where $X_t$ follows \eqref{eq:sde} with $X_0\sim\rho_0$.
    Then, the resulting FBSDEs system %
    takes the form: %
    \begin{subequations}
        \label{eq:fbsdet}
        \begin{empheq}[left={\text{\normalfont
            $\begin{array}{l} \text{FBSDEs} \\ \text{w.r.t. \eqref{eq:sde}} \end{array}$:
        }\empheqlbrace}]{align}
            \rd X_t &= \pr{f_t + \sigma Z_t} \dt + \sigma \rd W_t \label{eq:fsdet} \\
            \rd Y_t &= \pr{\frac{1}{2} \norm{Z_t}^2 + F_t } \dt + Z_t^\T \rd W_t
            \label{eq:bsdet1} \\
            \rd \Yhat_t &= \pr{\frac{1}{2} \norm{\Zhat_t}^2 +
                            \nabla \cdot (\sigma \Zhat_t -f_t) + \Zhat_t^\T Z_t - F_t
                        } \dt + \Zhat_t^\T \rd W_t
            \label{eq:bsdet2}
        \end{empheq}
    \end{subequations}
    Now, consider a similar transformation in \eqref{eq:sb2-pde} but instead w.r.t. the ``reversed'' SDE $\Xbar_s \sim $ \eqref{eq:rsde} and $\Xbar_0\sim\target$,
    \ie $Y_s \equiv Y(\Xbar_s,s) = \log \Psi(\Xbar_s, s)$, and \textit{etc}.
    The resulting FBSDEs system reads %
    \begin{subequations}
        \label{eq:fbsdes}
        \begin{empheq}[left={\text{\normalfont
            $\begin{array}{l} \text{FBSDEs} \\ \text{w.r.t. \eqref{eq:rsde}} \end{array}$:
        }\empheqlbrace}]{align}
            \rd \Xbar_s &= \pr{-f_s + \sigma \Zhat_s} \ds + \sigma \rd W_s \label{eq:fsdes} \\
            \rd Y_s &= \pr{\frac{1}{2} \norm{Z_s}^2 +
                            \nabla \cdot (\sigma Z_s +f_s) + Z_s^\T \Zhat_s {- F_s}
                        } \ds + Z_s^\T \rd W_s \label{eq:bsdes1} \\
            \rd \Yhat_s &= \pr{ \frac{1}{2} \norm{\Zhat_s}^2 {+ F_s} } \ds + \Zhat_s^\T \rd W_s
            \label{eq:bsdes2}
        \end{empheq}
    \end{subequations}

    Since $Y_t + \Yhat_t = \log\rho(X,t)$ by construction, %
    the functions $f_t$ and $F_t$ in \eqref{eq:fbsdet} take the arguments
    \begin{align*}
        f_t := f_t(X_t, \exp(Y_t+\Yhat_t)) \quad \text{ and } \quad F_t := F_t(X_t, \exp(Y_t+\Yhat_t)).
    \end{align*}
    Similarly, we have $f_s:= f_s(\Xbar_s, \exp(Y_s+\Yhat_s))$ and $F_s:= F_s(\Xbar_s, \exp(Y_s+\Yhat_s))$ in \eqref{eq:fbsdes}.
\end{theorem}
\begin{proof}
    See Appendix~\ref{sec:a3.thm}.
\end{proof}
Just like how \eqref{eq:sb2-pde} generalizes \eqref{eq:sb-pde},
our results in Theorem \ref{thm:mfg-fbsde} also generalize the ones appearing in vanilla SB-FBSDE \citep{chen2021likelihood}
(see \eqref{eq:prior} in Appendix \ref{sec:a2})
by introducing nontrivial MF interaction $F$.
Despite seemingly complex compared to the original PDEs \eqref{eq:sb2-pde},
these FBSDEs systems --- namely \eqref{eq:fbsdet} and \eqref{eq:fbsdes} ---
stand as the foundation for developing scalable numerical methods,
as they describe precisely
how the values of $Y\equiv\log\Psi$ and $\Yhat\equiv\log\Psihat$ shall change along the optimal SDEs
(notice, \eg that both $Y_t$ and $Z_t$ are functions of $X_t$ from \eqref{eq:nkc-yz}).
Essentially, the nonlinear FK Lemma provides a stochastic representation (in terms of $Y$ and $\Yhat$) of the PDEs in \eqref{eq:sb2-pde}
by expanding them w.r.t. the optimal SDEs in \eqref{eq:sb-sde} using the It\^{o} formula \citep{ito1951stochastic}.
Consequently, %
rather than solving the PDEs \eqref{eq:sb2-pde} in the \textit{entire function space} as in the prior work \citep{lin2021alternating},
it suffices to solve them \textit{locally around high probability regions characterized by} \eqref{eq:sb-sde},
which leads to computationally scalable methods.

\subsection{Design of the computational framework}  \label{sec:3.3}

Looking from Theorem \ref{thm:mfg-fbsde}, it suffices to approximate
$Y_\theta \approx Y$ and $\Yhat_\phi \approx \Yhat$
with some parametrized functions (we use DNNs), since one may infer
$Z_\theta \approx \sigma \nabla Y_\theta$ and $\Zhat_\phi \approx \sigma \nabla \Yhat_\phi$,
as suggested by \eqref{eq:nkc-yz},
and then solve for $(X_t,\Xbar_s)$ via (\ref{eq:fsdet}, \ref{eq:fsdes}).
Below, we explore options of designing training objectives for $(\theta,\phi)$,
with the aim to encourage $(Y_\theta,\Yhat_\phi)$ to satisfy the FBSDEs systems in (\ref{eq:fbsdet}, \ref{eq:fbsdes}).

\textbf{Option 1: $\Lipf$.}
Given how Theorem \ref{thm:mfg-fbsde} generalizes the one in \citep{chen2021likelihood}
(see \eqref{eq:prior} in Appendix~\ref{sec:a2}),
it is natural to wonder if adopting the computation used to derive \eqref{eq:L-ipf},
\eg $\Lipf(\phi) := \int \E[\rd Y_t^\theta + \rd \Yhat_t^\phi]$,\footnote{
    Additionally, we have $\Lipf(\theta) := \int \E[\rd Y_s^\theta + \rd \Yhat_s^\phi]$;
    see \eqref{eq:nll} in Appendix \ref{sec:a2} for the derivation.
}
suffices to reach the FBSDE \eqref{eq:fbsdet}.
This is, unfortunately, not the case as one can verify that
\begin{align*}
    \Lipf^{\text{\eqref{eq:fbsdet}}}(\phi) := \int \E\br{\rd Y_t^\theta + \rd \Yhat_t^\phi}
    = \int \E\br{\frac{1}{2} \norm{\Zhat_t^\phi + Z^t_\theta}^2 + \nabla\cdot(\sigma\Zhat_t^\phi - f)} \dt
    = \Lipf^{\text{\eqref{eq:L-ipf2}}}(\phi).
\end{align*}
Despite that \eqref{eq:fbsdet} differs from \eqref{eq:prior} by the extra terms ``$+F_t$'' in \eqref{eq:bsdet1} and ``$-F_t$'' in \eqref{eq:bsdet2}, the two terms cancel out in the sum of $\rd Y_t^\theta + \rd \Yhat_t^\phi$,
thereby yielding the same objectives that \emph{do not depend on $F$}.
This implies that naively optimizing $\Lipf$ from \citep{chen2021likelihood} is insufficient for solving FBSDE systems with nontrivial $F$.
We must seek additional objectives, if any, in order to respect the MF structure.

\textbf{Option 2: $\Lipf$ + Temporal Difference objective $\Ltd$.}
Let us revisit the relation between the FBSDEs (\ref{eq:fbsdet}, \ref{eq:fbsdes}) and their PDEs counterparts ---
but this time the HJB in \eqref{eq:mfg2-pde}.
Take $(X_t, Y_t)$ for example:
The fact that $Y_t = \log \Psi(X_t,t) = -u(X_t,t)$ %
suggests
an alternative interpretation of $Y_t$ as the
stochastic representation of the HJB,
which, crucially, can be seen as the continuous-time analogue of the Bellman equation \citep{bellman1954theory}.
Indeed, discretizing \eqref{eq:bsdet1} with some fixed step size $\delta t$ yields
\begin{align}
    Y_{t+\delta t}^\theta = Y_t^\theta +
      \markcc{\pr{\frac{1}{2} \norm{Z_t^\theta}^2 + { F_t}}\delta t}
    + \markbb{{Z_t^\theta}^\T \delta W_t},
    \quad \delta W_t \sim \calN(\mathbf{0},  \delta t \mI),
    \label{eq:bellman}
\end{align}
which resembles a (non-discounted) Temporal Difference (TD) \citep{todorov2009efficient,lutter2021value}
except that, in addition to the standard \markcc{``\textit{rewards}'' (in terms of control and state costs)},
we also have a \markbb{stochastic term}.
This stochastic term,
which vanishes in the vanilla Bellman equation upon taking expectations,
plays a crucial role %
in characterizing the inherited stochasticity of the value function $Y_t$.
With this interpretation in mind, we can construct suitable TD targets for our FBSDEs systems as shown below.
\begin{proposition}[TD objectives $\Ltd$ for (\ref{eq:fbsdet}, \ref{eq:fbsdes})] \label{prop:td}
    The single-step TD targets take the forms:
    \begin{subequations}
    \label{eq:td-single}
    \begin{align}
        \TDhat_{t+\delta t}^\text{single} :=& \Yhat^\phi_t + \markcc{\pr{
                \frac{1}{2} \norm{\Zhat_t^\phi}^2 + \nabla \cdot (\sigma \Zhat_t^\phi -f_t) + \Zhat_t^\phi{}^\T Z_t^\theta {- F_t}}
            \delta t} + \markbb{\Zhat_t^\phi{}^\T \delta W_t}, %
        \label{eq:td2-single}
            \\
        \TD_{s+\delta s}^\text{single} :=& Y_s^\theta + \markcc{\pr{
                \frac{1}{2} \norm{Z_s^\theta}^2 + \nabla \cdot (\sigma Z_s^\theta + f_s) + Z_s^\theta{}^\T \Zhat_s^\phi {- F_s}}
            \delta s} + \markbb{Z_s^\theta{}^\T \delta W_s}, %
        \label{eq:td1-single}
    \end{align}
    \end{subequations}
    with
    $\TDhat_0 := \log \rho_0 - Y_0^\theta$ and
    $\TD_0 := \log \target - \Yhat_0^\phi$,
    and the multi-step TD targets take the forms:
    \begin{align}
        \TDhat^\text{multi}_{t+\delta t} :=
            \TDhat_0 + \sum_{\tau=\delta t}^t \delta \Yhat_\tau,
            \qquad
            \TD^\text{multi}_{s+\delta s} :=
            \TD_0    + \sum_{\tau=\delta s}^s \delta Y_\tau,
        \label{eq:td-multi}
    \end{align}
    where
    $\delta\Yhat_t := \TDhat_{t+\delta t}^\text{single} - \Yhat_t$ and
    $\delta Y_s := \TD_{s+\delta s}^\text{single} - Y_s$.
    Given these TD targets, we can construct
    \begin{align}
        \Ltd(\theta) = \sum_{s=0}^T \E \br{ \norm{
            Y_\theta(\Xbar_s, s) - \TD_s
        }} \delta s, \text{ }
        \Ltd(\phi) = \sum_{t=0}^T \E \br{ \norm{
            \Yhat_\phi(X_t, t) - \TDhat_t
        }} \delta t.
        \label{eq:L-td}
    \end{align}
\end{proposition}
\begin{proof}
    See Appendix~\ref{sec:a3.prop-td}.
\end{proof}

\newcommand*\Xon{ \bm{X}_\text{on}^\theta }%
\newcommand*\Xoff{ \bm{X}_\text{off}^\theta }%
\newcommand*\Xbaron{ \bm{\Xbar}_\text{on}^\phi }%
\newcommand*\Xbaroff{ \bm{\Xbar}_\text{off}^\phi }%

\begin{figure}[t]
    \vskip -0.1in
    \begin{algorithm}[H]
    \small
    \caption{\small Deep Generalized \SB (DeepGSB)}
    \label{alg:jacobi}
    \begin{algorithmic}
        \STATE {\bfseries Input:}
            ($Y_\theta$, $\Yhat_\phi$, $\sigma \nabla Y_\theta$, $\sigma \nabla \Yhat_\phi$) for critic or
            ($Y_\theta$, $\Yhat_\phi$, \markgreen{$Z_\theta$, $\Zhat_\phi$}) for \markgreen{actor-critic} parametrization.
        \REPEAT
            \STATE Sample
                $\bm{X}^\theta \equiv \{X_t^\theta, Z_t^\theta, \delta W_t \}_{t\in[0,T]}$ from the forward SDE \eqref{eq:fsdet}; add $\bm{X}^\theta$ to replay buffer $\calB$.
            \FOR{$k=1$ {\bfseries to} $K$ }
                \STATE Sample on-policy $\Xon$ and off-policy $\Xoff$ samples respectively
                        from $\bm{X}^\theta$ and $\calB$.
                \STATE Compute $\calL(\phi) = \Lipf(\phi; \Xon) + \Ltd(\phi; \Xon) + \Ltd(\phi; \Xoff)~\markgreen{ +\calL_\text{FK}(\phi; \Xon)}$.
                \STATE Update $\phi$ with the gradient $ \nabla_\phi \calL(\phi)$.
            \ENDFOR
            \STATE Sample
                $\bm{\Xbar}^\phi \equiv \{\Xbar_s^\phi, \Zhat_s^\phi, \delta W_s \}_{s\in[0,T]}$ from the backward SDE \eqref{eq:fsdes}; add $\bm{\Xbar}^\phi$ to replay buffer $\bar{\calB}$.
            \FOR{$k=1$ {\bfseries to} $K$ }
                \STATE Sample on-policy $\Xbaron$ and off-policy $\Xbaroff$ samples respectively
                    from $\bm{\Xbar}^\phi$ and $\bar{\calB}$.
                \STATE Compute $\calL(\theta) = \Lipf(\theta; \Xbaron) + \Ltd(\theta; \Xbaron) + \Ltd(\theta; \Xbaroff) ~\markgreen{ +\calL_\text{FK}(\theta; \Xbaron)}$.
                \STATE Update $\theta$ with the gradient $ \nabla_\theta \calL(\theta)$.
            \ENDFOR
        \UNTIL{ converges }
    \end{algorithmic}
    \end{algorithm}
\vskip -0.2in
\end{figure}

It can be readily seen that the single-step TD targets in \eqref{eq:td-single} obey a similar structure to \eqref{eq:bellman}, except deriving from different SDEs
(\ref{eq:bsdet2}, \ref{eq:bsdes1}). %
Doing so reduces the computational overhead, as
the related objectives for each parameter, \eg $\Lipf(\theta)$ and $\Ltd(\theta)$,
can be evaluated from the same expectation.
In practice, we find that the multi-step objectives often yield better performance, as consistently observed in the DeepRL literature
\citep{meng2021effect,van2016effective,hessel2018rainbow}.
Additionally, common practices such as computing $\TDhat_t$ and $\TD_s$ using the exponential moving averaging (\ie target values) and replay buffers also help stabilize training.
Finally, the fact that the TD targets in \eqref{eq:L-td} appear as the regressands implies
that from a computational standpoint, the MF interaction $F$ needs \textit{not} to be continuous or differentiable.

\textbf{Necessity and sufficiency of $\Lipf + \Ltd$.} %
It remains unclear whether appending $\Ltd$ to the objective suffices for $(Y_\theta,\Yhat_\phi)$ to satisfy the FBSDEs (\ref{eq:fbsdet}, \ref{eq:fbsdes}).
Below, we provide a positive result.
\begin{proposition} \label{prop:sufficient}
    The functions
    $(Y_\theta, Z_\theta, \Yhat_\phi, \Zhat_\phi)$ satisfy the FBSDEs (\ref{eq:fbsdet},\ref{eq:fbsdes}) in
    Theorem~\ref{thm:mfg-fbsde} if and only if they are the minimizers of the combined losses $\calL(\theta,\phi) := \Lipf(\phi) + \Ltd(\phi) + \Lipf(\theta) + \Ltd(\theta)$.
\end{proposition}
\begin{proof}
    See Appendix~\ref{sec:a3.prop-conv}.
\end{proof}
Proposition~\ref{prop:sufficient} asserts the validity of the combined objectives $\Lipf + \Ltd$ in solving the generalized SB-FBSDEs in Theorem~\ref{thm:mfg-fbsde}, and hence the MFG problem in \eqref{eq:mfg2-pde}.
It shall be interpreted as follows: The minimizer of $\Lipf$, as implied in Lemma~\ref{lemma:kl-sbfbsde}, would always establish a valid ``bridge'' transporting between the boundary distributions $\rho_0$ and $\target$; yet, without further conditions, this bridge needs not obey a ``Schr{\"o}dinger'' bridge. While general IPF and Sinkhorn \citep{chen2015optimal,chen2021density}, upon proper initialization or discretization, provides one way to ensure the convergence toward the ``S''B, our Proposition~\ref{prop:sufficient} suggests an alternative by introducing the TD objectives $\Ltd$. This gives us flexibility to handle generalized SB in MFGs where $F$ becomes nontrivial or non-convex. Further, it naturally handles non-differentiable $F$, which can offer extra benefits in many cases.

\textbf{Option 3: $\Lipf$ + $\Ltd$ + FK objective $\calL_\text{FK}$.}
Though it seems sufficient to parametrize $(Y_\theta, \Yhat_\phi)$ then infer $Z_\theta := \sigma \nabla Y_\theta$ and $\Zhat_\phi := \sigma \nabla \Yhat_\phi$, as suggested in previous options,
in practice we find that parametrizing $(Z_\theta, \Zhat_\phi)$ with two additional DNNs
then imposing the following FK objective, \ie
\begin{align*}
    \calL_\text{FK}(\theta) = \sum_{s=0}^T \E \br{ \norm{
        \sigma \nabla Y_\theta(\Xbar_s, s) - Z_\theta(\Xbar_s, s)
    }} \delta s, \text{ }
    \calL_\text{FK}(\phi) = \sum_{t=0}^T \E \br{ \norm{
        \sigma \nabla \Yhat_\phi(X_t, t) - \Zhat_\phi(X_t, t)
    }} \delta t,
\end{align*}
often offers extra robustness.
These objectives aim to ensure that the nonlinear FK \eqref{eq:nkc-yz} holds.

Our \textbf{DeepGSB} is summarized in Alg.~\ref{alg:jacobi}.
Hereafter, we refer Option 2 and 3 respectively to \textit{DeepGSB critic} and \textit{DeepGSB actor-critic},
as $Y_\theta$ and $Z_\theta$ play similar roles of critic and actor networks \citep{mnih2013playing,lillicrap2015continuous}.

\textbf{Remarks on convergence.}
Despite Alg.~\ref{alg:jacobi} sharing a similar alternating structure to IPF \citep{chen2021likelihood,de2021diffusion,chen2021density},
the combined objective, \eg
$\calL(\phi) \propto \KL(\rho^\theta || \rho^\phi) + \E_{\rho^\theta}[\Ltd(\phi)] \neq \KL(\rho^\phi || \rho^\theta)$
is \emph{not} equivalent to the (reversed) KL appearing in IPF.
Instead,
DeepGSB may be closer to
trust region optimization \citep{schulman2015trust},
as both iteratively update the policy using samples from the previous stage while subjected to some KL penalty: $\pi^{(i+1)} = \argmin_\pi  \KL(\pi^{(i)} || \pi) + \E_{\pi^{(i)}} [\mathcal{L}(\pi)]$.
Hence, one can expect DeepGSB to admit similar
monotonic improvement and local convergence properties.
We leave more discussions to Appendix \ref{sec:a4.conv}.

\section{Experiment} \label{sec:4}

\newcommand*\GMM{ \markaa{GMM} }%
\newcommand*\Vneck{ \markbb{V-neck} }%
\newcommand*\Stunnel{ \markcc{S-tunnel} }%

\textbf{Instantiation of MFGs.$\quad$}
We validate our {DeepGSB} on two classes of MFGs, %
including classical crowd navigation ($d$=2) and high-dimensional ($d$=1000) opinion depolarization.
For crowd navigation, we consider three MFGs appearing in prior methods \citep{ruthotto2020machine,lin2021alternating}, including
\markaa{\textit{(i)}} asymmetric obstacle avoidance,
\markbb{\textit{(ii)}} entropy interaction with a V-shape bottleneck,
and  \markcc{\textit{(iii)}} congestion interaction on an S-shape tunnel.
We will refer to them respectively as \markaa{GMM}, \markbb{V-neck}, and \markcc{S-tunnel}.
The obstacles and the initial/target Gaussian distributions $(\rho_0,\target)$ are shown in Fig.~\ref{fig:mfg}.
For opinion depolarization,
we set $\rho_0$ and $\target$ to two zero-mean Gaussians with varying variances
for representing the initially polarized and desired moderated opinion distributions.
Finally,
we consider zero and constant base drift $f$ respectively for GMM and V-neck/S-tunnel,
and adopt the polarized MF dynamics \citep{gaitonde2021polarization}
for opinion MFG;
see Sec.~\ref{sec:4.2} for a detailed discussion.

\begin{figure}[H]
    \vskip -0.05in
    \centering
    \begin{minipage}{0.46\textwidth}
        \centering
        \includegraphics[width=\textwidth]{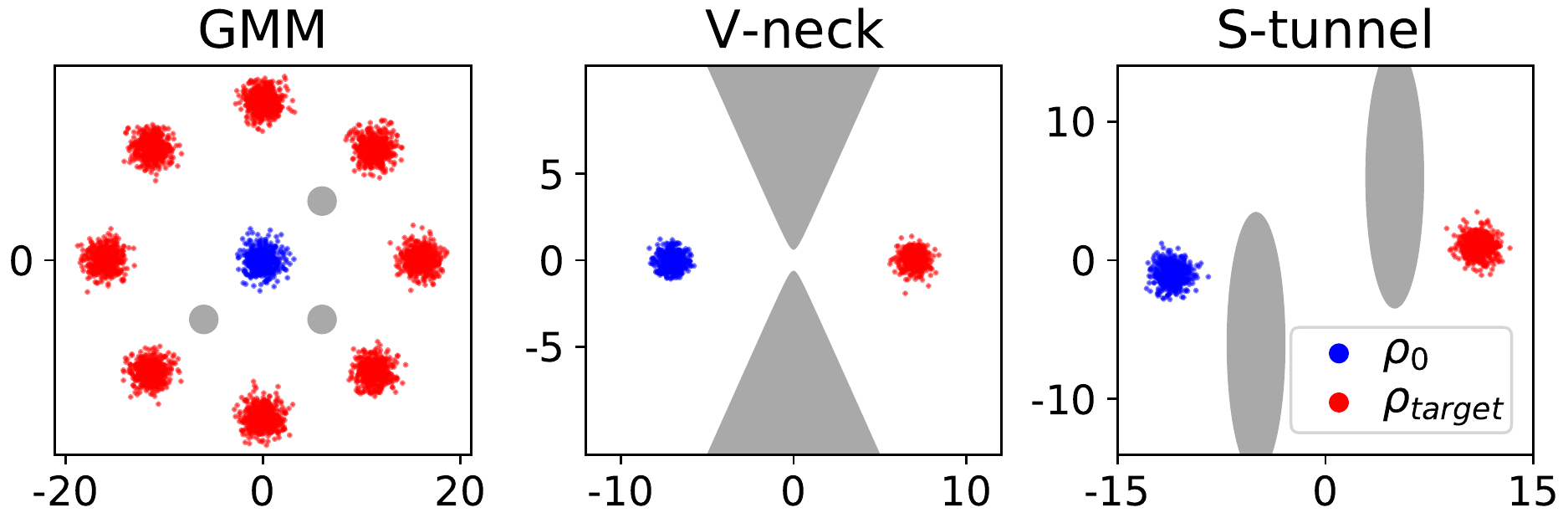}
        \vskip -0.05in
        \caption{
            Crowd navigation MFGs.
        }
        \label{fig:mfg}
    \end{minipage}
    $\quad$
    \begin{minipage}{0.48\textwidth}
          \centering
          \captionsetup{type=table}
          \captionsetup{justification=centering}
          \caption{
              MF interactions for 3 crowd navigation MFGs and the high-dimensional opinion MFG.
          }
          \label{table:2}
          \vskip -0.05in
          \centering
          \scalebox{0.95}{
          \begin{tabular}{cc}
            \toprule
          \GMM ($d$=2)      & $F_\text{obstacle}$ \\
          \Vneck ($d$=2)   & $F_\text{obstacle} + F_\text{entropy}$ \\
          \Stunnel ($d$=2) & $F_\text{obstacle} + F_\text{congestion}$ \\
          {Opinion} ($d$=1000)  & $F_\text{entropy}$ \\
            \bottomrule
          \end{tabular}}
    \end{minipage}
    \vskip -0.1in
\end{figure}

\textbf{MF interactions $F$.$\quad$}
We follow standard treatments from the MFG theory \citep{lasry2007mean} by noting that
given a functional $\calF(\rho)$ that quantifies the MF cost w.r.t. the population $\rho$,
\eg $\calF_\text{entropy} := \E_\rho[\log\rho]$ or $\calF_\text{congestion} := \E_{x,y\sim\rho} [\frac{1}{\norm{x-y}^2+1}]$,
one can derive its associated MF interaction function $F(x,\rho)$ by taking the functional derivative, \ie
$\frac{\delta\calF(\rho)}{\delta\rho}(x) = F(x,\rho)$.
Hence, the entropy and congestion MF interactions,
together with the obstacle cost, follow (see Appendix~\ref{sec:a4.2} for the derivation):
\begin{align}
        F_\text{entropy} := \log \rho(x,t) {+} 1, \quad
        F_\text{congestion} := \E_{y\sim\rho}\br{\frac{2}{\norm{x{-}y}^2{+}1}}, \quad
        F_\text{obstacle} := 1500 {\cdot} \mathbbm{1}_\text{obs}(x),
    \label{eq:F}
\end{align}
where $\mathbbm{1}_\text{obs}(\cdot)$ is the (discontinuous) indicator of the problem-dependent obstacle set.
We summarize the MF interaction in Table~\ref{table:2}.

\textbf{Architecture \& Hyperparameters.$\quad$}
We parameterize the functions with fully-connected DNNs for crowd navigation,
and deep residual networks for high-dimensional opinion MFGs.
All networks adopt sinusoidal time embeddings and are trained with AdamW \citep{loshchilov2017decoupled}.
All SDEs in (\ref{eq:fbsdet}, \ref{eq:fbsdes}) are solved with the Euler-Maruyama method.
Due to space constraints, we will focus mostly on the results of \textbf{a}ctor-\textbf{c}ritic parametrization \textbf{DeepGSB\nobreakdash-ac},
and leave the discussion of \textbf{c}ritic parametrization \textbf{DeepGSB\nobreakdash-c},
along with additional experimental details, to Appendix \ref{sec:a5}.

\subsection{Two-dimensional crowd navigation} \label{sec:4.1}

Figure \ref{fig:crowd-nav} shows the simulation results of our \textbf{DeepGSB-ac} on three crowd navigation MFGs.
We also report existing numerical methods \citep{ruthotto2020machine, lin2021alternating, chen2021density}
that are best-tuned on each MFG (see Appendix \ref{sec:a5.1} for details)
but note that in practice,
they either
require softening $F$ to be  differentiable \citep{ruthotto2020machine, lin2021alternating} to yield reasonable results,
or
discretizing the state space \citep{chen2021density}, which can lead to prohibitive complexity.\footnote{
    As stated in \citep{chen2021density}, the complexity scales \textit{quadratically} w.r.t. the number of discretized grid points.
}

We first compare to \citet{chen2021density} on\GMM(see Fig.~\ref{fig:gmm}) as
their method only applies
to non-MF interaction, \ie
$F:= F(x)$. %
While \textbf{DeepGSB-ac} guides the population to smoothly avoid all obstacles
(notice the sharp contours of $Y$ around them),
\citep{chen2021density} struggles to escape due to the discretization of the state space (hence the policy).
To better examine the effect of $\rho$ in $F(x,\rho)$,
we next simulate the dynamics on\Vneck(see Fig.~\ref{fig:bottle}) with and without the MF interaction.
It is clear that our \textbf{DeepGSB\nobreakdash-ac} encourages
the population to spread out once the entropy interaction $F_\text{entropy}$ is enabled,
yet a similar effect is barely observed in \citep{ruthotto2020machine}.
We observed difficulties in balancing
the MF interaction $F$ and the terminal penalty $\KL(\rho_T||\target)$ for \citep{ruthotto2020machine},
yet this problem is alleviated in DeepGSB by construction.
Lastly, we validate the robustness of our method on\Stunnel(see Fig.~\ref{fig:elp}) w.r.t. varying diffusions $\sigma=\{0.5,1,2\}$.
Again,
our \textbf{DeepGSB\nobreakdash-ac} reaches the same $\target$
despite being subject to different levels of stochasticity.
This is in contrast to \citep{lin2021alternating},
which, due to discarding the SDE dynamics,
necessitates solving PDEs on the entire state space
that may be sensitive to hyperparameters.
In short, our DeepGSB outperforms prior methods
\citep{ruthotto2020machine, lin2021alternating, chen2021density}
by better respecting obstacles and MF interactions yet without losing convergence to $\target$,
and its performance remains robust across different MFGs.

\begin{figure}
    \vskip -0.1in
    \centering
    \subfloat{
        \includegraphics[width=\textwidth]{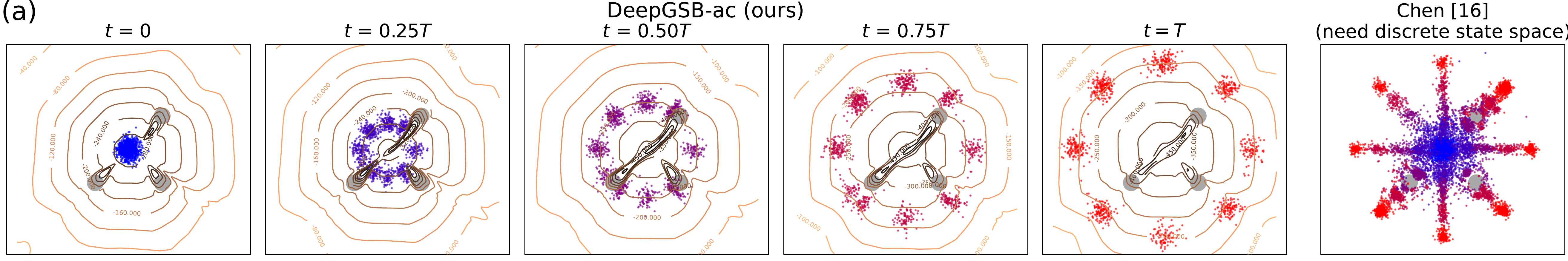}
        \label{fig:gmm}
    }\\
    \vspace{-7.5pt}
    \subfloat{
        \includegraphics[width=\textwidth]{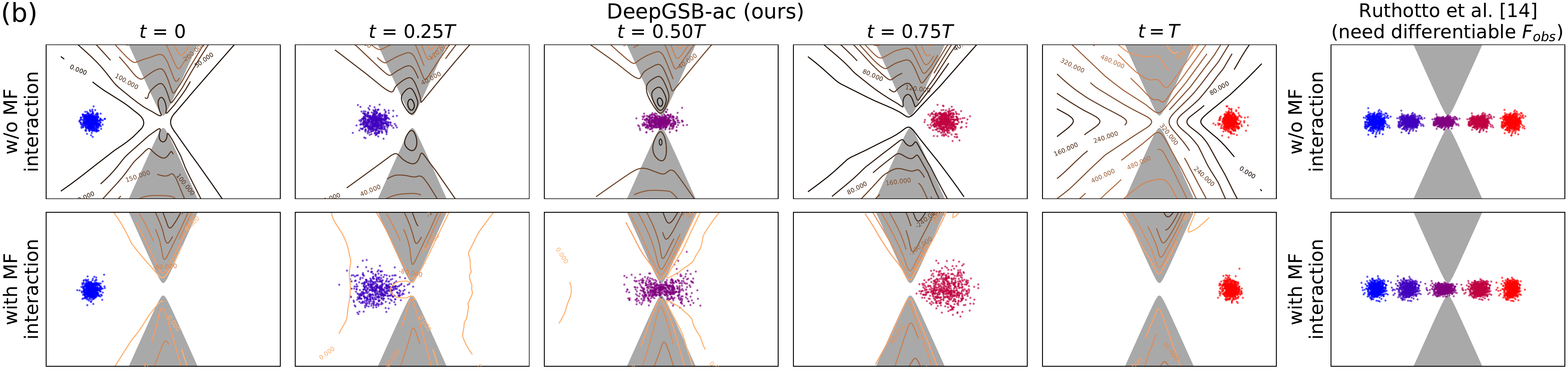}
        \label{fig:bottle}
    }\\
    \vspace{-7.5pt}
    \subfloat{
        \includegraphics[width=\textwidth]{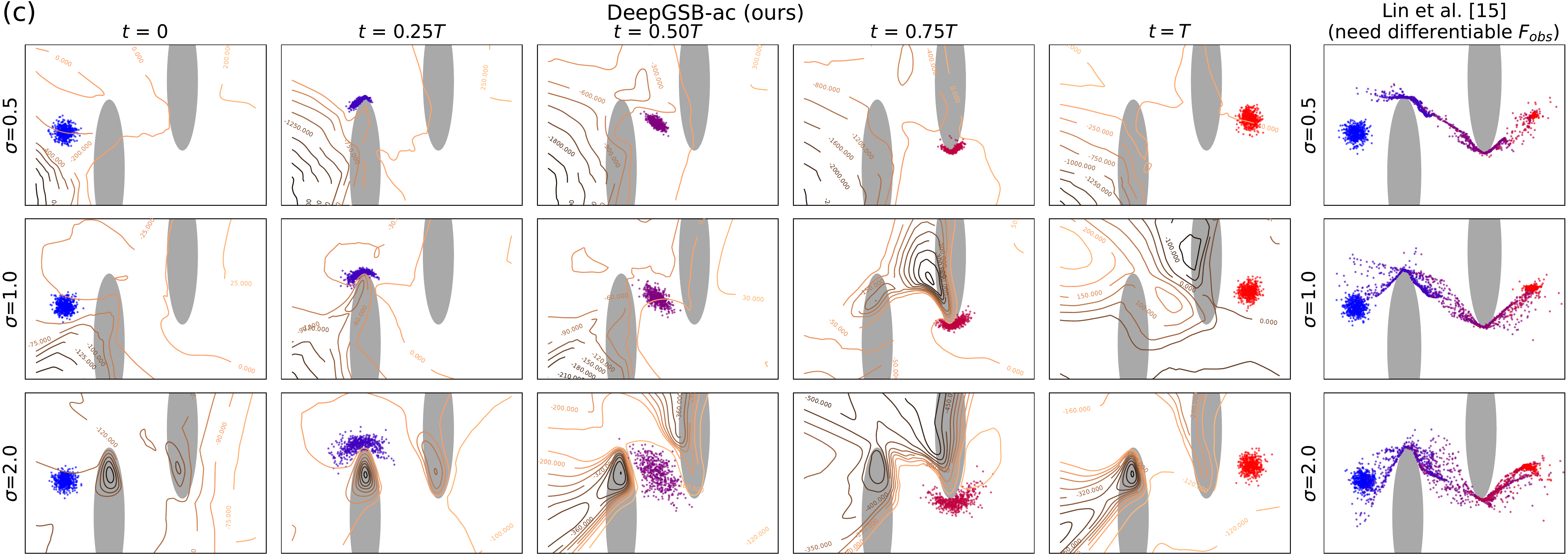}
        \label{fig:elp}
    }
    \caption{
        Simulation of the three crowd navigation MFGs, including
        (a) \markaa{GMM}, (b) \markbb{V-neck}, and (c) \markcc{S-tunnel},
        from $t=0$ to $T$.
        The first five columns show the population snapshots, each with a different color,
        guided by our \textbf{DeepGSB-ac},
        whereas the sixth (rightmost) column overlays the same population snapshots generated by
        existing methods \citep{ruthotto2020machine, lin2021alternating, chen2021density}.
        The time-varying contours represent
        $Y_\theta \approx \log \Psi$ whose gradient relates to the policy via $Z = \sigma\nabla Y$.
        This figure is best viewed in color.
    }
    \label{fig:crowd-nav}
\end{figure}

\begin{figure}[t]
    \vskip -0.15in
    \centering
    \subfloat{
        \includegraphics[height=3.65cm]{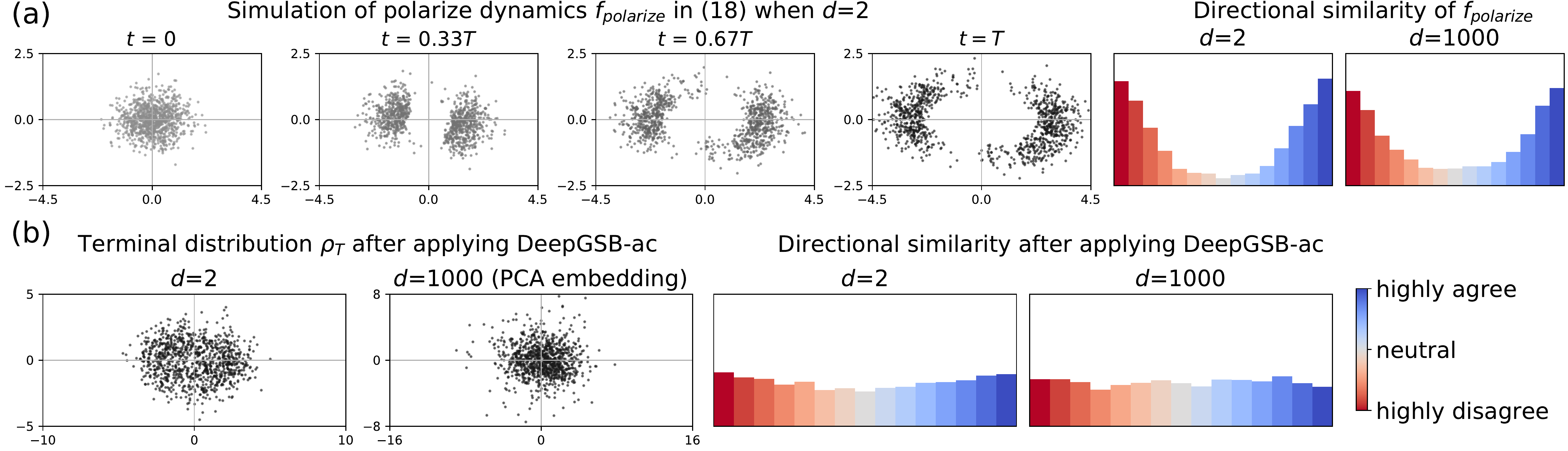}
        \label{fig:opinion-a}%
    }
    \subfloat{
        \textcolor{white}{\rule{1pt}{1pt}}
        \label{fig:opinion-b}%
    }
    \vskip -0.05in
    \caption{
        (a) Visualization of polarized dynamics $\bar{f}_\text{polarize}$ in 2- and 1000-dimensional opinion space,
        where the \textit{directional similarity} \citep{schweighofer2020agent} counts the histogram of cosine angle between pairwise opinions
        at the terminal distribution $\rho_T$.
        (b) \textbf{DeepGSB-ac} guides $\rho_T$ toward moderated distributions,
        hence \textit{depolarizes} the opinion dynamics.
        We use the first two principal components to visualize $d$=1000.
    }
    \label{fig:opinion}
    \vskip -0.1in
\end{figure}

\subsection{High-dimensional opinion depolarization} \label{sec:4.2}

Next, we showcase our {DeepGSB} %
in solving high-dimensional MFGs in the application of opinion dynamics
\citep{schweighofer2020agent,gaitonde2021polarization,hkazla2019geometric},
where each agent now possesses a $d$-dimensional opinion $x\in\sR^d$
that evolves through the interactions with the population.
In light of increasing recent attention,
we consider a particular class of opinion dynamics %
known to yield strong \textit{polarization} \citep{gaitonde2021polarization},
\ie the agents' opinions tend to partition into groups holding diametric views.
Take the \textit{party model} \citep{gaitonde2021polarization} for instance:
Given a random information $\xi\in\sR^d$ sampled from some distribution independent of $\rho$,
each agent updates the opinion following a normalized polarize dynamic
$\bar{f}_\text{polarize}=\nicefrac{f_\text{polarize}}{\norm{f_\text{polarize}}^{\frac{1}{2}}}$, where
\begin{align}
    f_\text{polarize}(x,\rho;\xi) := {\E_{y\sim\rho}\br{a(x,y;\xi)\bar{y}}},
    \text{ }\text{ } a(x,y;\xi) :=
    \begin{cases}
        1  & \text{if } \sign(\langle x,\xi \rangle) = \sign(\langle y,\xi \rangle) \\
        -1 & \text{otherwise}
    \end{cases},
    \label{eq:opinion-f}
\end{align}
and $\bar{y} = \nicefrac{y}{\norm{y}^{\frac{1}{2}}}$.
The \textit{agreement} function $a(x,y;\xi)$
indicates whether the two opinions $x$ and $y$ agree on the information $\xi$.
Intuitively,
the dynamic in \eqref{eq:opinion-f} suggests that the agents
tend to be receptive to opinions they agree with, and antagonistic to opinions they disagree with.
As shown in Fig.~\ref{fig:opinion-a},
this behavioral assumption, also known as biased assimilation \citep{lord1979biased,dandekar2013biased},
can easily lead to polarization.

We can apply our MFG framework \eqref{eq:mfg2-pde} to this polarized base drift \eqref{eq:opinion-f},
where, starting from some weakly polarized $\rho_0$,
we seek a policy that compensates the polarization tendency and helps guide the opinion towards
a moderated distribution $\target$ (assuming as Gaussian for simplicity).
We consider the entropy MF interaction $F_\text{entropy}$
as it encourages opinions diversity before reaching consensus.
As shown in Fig.~\ref{fig:opinion-b},
in both lower- ($d$=2) and higher- ($d$=1000) dimensions,
our \textbf{DeepGSB-ac} successfully guides the opinion towards the desired distribution centered symmetrically at $\mathbf{0}\in \sR^d$, thereby mitigates the polarization.
Results of {DeepGSB-c} remain similar despite being more sensitive to hyperparameters; see Appendix~\ref{sec:a5.2}.
We highlight these state-of-the-art results on a challenging class of MFGs that,
comparing to existing methods \citep{ruthotto2020machine,lin2021alternating},
consider a more difficult mean-field dynamic ($f(x,\rho)$ \textit{vs.} $f(x)$)
in an order of magnitude higher dimension ($d$=1000 \textit{vs.} $d$=100).

\subsection{Discussion} \label{sec:5}

\begin{wrapfigure}[13]{r}{0.63\textwidth}
    \vspace{-18pt}
    \captionsetup{type=table}
    \caption{
        Comparison of DeepGSB-ac \textit{vs.} DeepGSB-c w.r.t
        Wasserstein distance to $\target$ and FBSDEs violation, in terms of
        TD errors and nonlinear FK, averaged over 3 runs.
    }
    \setlength\tabcolsep{4.5pt}
    \label{table:3}
    \vskip -0.05in
    \centering
    \scalebox{0.98}{
    \begin{tabular}{ccrrrr}
        \toprule
        \multirow{2}{*}{MFGs}              & \multirow{2}{*}{{DeepGSB}}   & \multirow{2}{*}{$\calW_2 \downarrow $} & \multicolumn{3}{c}{FBSDEs Violation $\downarrow$ }   \\
                                            &                           &                                           & {  $\Ltd(\phi)$}  & {  $\Ltd(\theta)$}    & { $\Lfk(\theta)$}          \\
        \midrule
        \multirow{2}{*}{\GMM}            & {-ac} & $\text{\textbf{.27}}{\scriptsize \text{$\pm$.16}}$   & $\text{{9.5}}{\scriptsize\text{$\pm$2.5}}$         & $\text{\textbf{7.1}}{\scriptsize\text{$\pm$0.6}}$  & $\text{{5.2}}{\scriptsize\text{$\pm$1.1}}$         \\
                                         & {-c}  & $\text{{.61}}{\scriptsize\text{$\pm$.91}}$           & $\text{\textbf{7.0}}{\scriptsize\text{$\pm$1.3}}$  & $\text{{10.1}}{\scriptsize\text{$\pm$1.6}}$        & $\text{\textbf{0.0}}{\scriptsize\text{$\pm$0.0}}$  \\
        \midrule
        \multirow{2}{*}{\Vneck}         & {-ac}  & $\text{\textbf{.00}}{\scriptsize\text{$\pm$.00}}$   & $\text{\textbf{4.9}}{\scriptsize\text{$\pm$1.5}}$  & $\text{\textbf{4.1}}{\scriptsize\text{$\pm$0.5}}$  & $\text{{0.6}}{\scriptsize\text{$\pm$0.2}}$         \\
                                        & {-c}   & $\text{{.01}}{\scriptsize\text{$\pm$.00}}$          & $\text{{8.2}}{\scriptsize\text{$\pm$0.8}}$         & $\text{{8.7}}{\scriptsize\text{$\pm$1.6}}$         & $\text{\textbf{0.0}}{\scriptsize\text{$\pm$0.0}}$  \\
        \midrule
        \multirow{2}{*}{\markcc{S-tunnel}} & {-ac} & $\text{\textbf{.01}}{\scriptsize\text{$\pm$.00}}$ & $\text{\textbf{25.5}}{\scriptsize\text{$\pm$2.3}}$ & $\text{{28.6}}{\scriptsize\text{$\pm$3.6}}$        & $\text{{2.1}}{\scriptsize\text{$\pm$0.1}}$         \\
                                           & {-c} & $\text{{.03}}{\scriptsize\text{$\pm$.01}}$         & $\text{{30.9}}{\scriptsize\text{$\pm$6.9}}$        & $\text{\textbf{26.4}}{\scriptsize\text{$\pm$5.5}}$ & $\text{\textbf{0.0}}{\scriptsize\text{$\pm$0.0}}$  \\
        \bottomrule
    \end{tabular}
    }
\end{wrapfigure}
\textbf{DeepGSB-ca \textit{vs.} DeepGSB-c.}
Table~\ref{table:3} compares actor-critic with critic parametrizations on crowd navigation MFGs.
While DeepGSB\nobreakdash-ac typically achieves lower Wasserstein and TD errors,
it seldom closes the consistency gap of $\Lfk(\theta)$,
as opposed to DeepGSB\nobreakdash-c.
In practice,
the results of DeepGSB\nobreakdash-c are visually indistinguishable from DeepGSB\nobreakdash-ac,
despite the different contours of $Y$;
see Appendix \ref{sec:a5.2} for more discussions.

\textbf{DeepGSB works with intractable $\target$.}
While the availability of the target density $\target$ is a common assumption adopted in prior works \citep{ruthotto2020machine, lin2021alternating}, in which $\target$ is involved in computing the boundary loss, in most real-world applications, $\target$ is seldom available.
Here, we show that DeepGSB works well without knowing $\target$ (and $\rho_0$) so long as we can sample from $X_0 \sim \rho_0$ and $\Xbar_0 \sim \target$. This is similar to the setup of generative modeling \citep{chen2021likelihood}.
In Fig.~\ref{fig:c}, we show that DeepGSB trained without the initial and terminal densities can converge equally well. Crucially, this is because DeepGSB replies on a variety of other mechanisms (\eg self-consistency in single-step TD objectives and KL-matching in IPF objective) to generate equally informative gradients. This is in contrast to \citep{ruthotto2020machine, lin2021alternating} where the training signals are mostly obtained by differentiating through $\KL(\rho||\target)$; consequently, their methods fail to converge in the absence of $\target$.

\begin{figure}[H]
    \vskip -0.1in
    \centering
    \includegraphics[width=\textwidth]{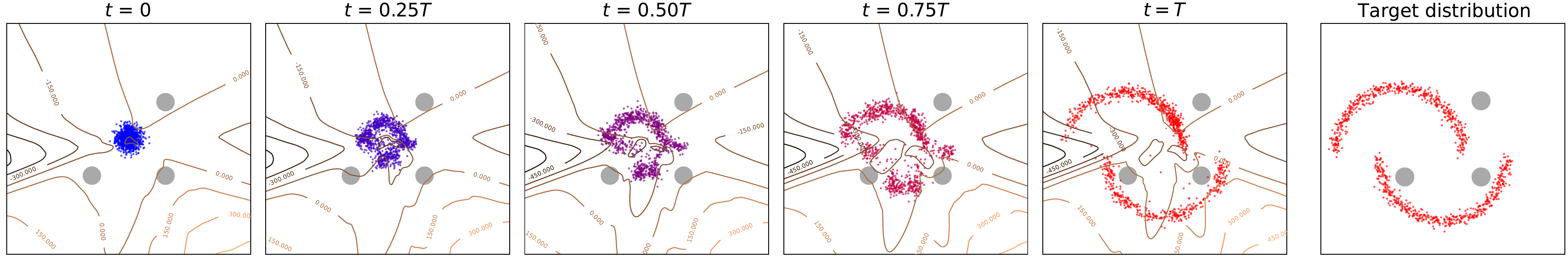}
    \vskip -0.05in
    \caption{
        DeepGSB-ac trained \emph{without} access to the initial and target distributions, \ie without $\TD_0$ and $\TDhat_0$.
        In this case, we compute $\Ltd$ with the single-step formulation in \eqref{eq:td-single}.
    }
    \label{fig:c}
    \vskip -0.1in
\end{figure}

\section{Conclusion and Limitation} \label{sec:6}

We present \textbf{DeepGSB}, a new numerical method for solving a challenging class of MFGs with distributional boundary constraints.
By generalizing prior FBSDE theory for \SB to accepting mean-field interactions,
we show that practical training can be achieved via an intriguing algorithmic connection to DeepRL.
Our DeepGSB outperforms prior methods in crowd navigation MFGs and
sets a new state-of-the-art record in depolarizing {1000}-dimensional opinion MFGs.

DeepGSB
is mainly developed for MFGs in \textit{unconstrained} state spaces such as $\sR^d$.
Yet,
it may be necessary
to adopt domain-specific structures, \eg \textit{constrained} state spaces.
  Additionally, the divergence in the IPF objectives may scale unfavorably as the dimension grows.
  This may be mitigated by adopting a simpler regression from \citet{de2021diffusion}.
  We leave this as a promising future direction.

\newpage

\section*{Boarder Impact}
Study of Mean-Field Games (MFGs) possesses its own societal influence.
Thus, as a MFG solver,
DeepGSB may pose a potential impact in offering %
solutions to previously unsolvable MFGs under more practical settings,
thereby facilitating new understanding of population behavior.

\begin{ack}
  The authors would like to thank Yu-ting Chiang, Augustinos and Molei for their helpful supports and kind discussion. The authors would also like to thank the anonymous Reviewer A4H3 for his/her initially harsh yet constructive comments on OpenReview, which led to substantial improvements of the theoretical results during rebuttal.
\end{ack}

\bibliographystyle{unsrtnat} %
\bibliography{reference.bib}

\begin{thebibliography}{71}
\providecommand{\natexlab}[1]{#1}
\providecommand{\url}[1]{\texttt{#1}}
\expandafter\ifx\csname urlstyle\endcsname\relax
  \providecommand{\doi}[1]{doi: #1}\else
  \providecommand{\doi}{doi: \begingroup \urlstyle{rm}\Url}\fi

\bibitem[Achdou et~al.(2014)Achdou, Buera, Lasry, Lions, and
  Moll]{achdou2014partial}
Yves Achdou, Francisco~J Buera, Jean-Michel Lasry, Pierre-Louis Lions, and
  Benjamin Moll.
\newblock Partial differential equation models in macroeconomics.
\newblock \emph{Philosophical Transactions of the Royal Society A:
  Mathematical, Physical and Engineering Sciences}, 372\penalty0
  (2028):\penalty0 20130397, 2014.

\bibitem[Achdou et~al.(2022)Achdou, Han, Lasry, Lions, and
  Moll]{achdou2022income}
Yves Achdou, Jiequn Han, Jean-Michel Lasry, Pierre-Louis Lions, and Benjamin
  Moll.
\newblock Income and wealth distribution in macroeconomics: A continuous-time
  approach.
\newblock \emph{The review of economic studies}, 89\penalty0 (1):\penalty0
  45--86, 2022.

\bibitem[Schweighofer et~al.(2020)Schweighofer, Garcia, and
  Schweitzer]{schweighofer2020agent}
Simon Schweighofer, David Garcia, and Frank Schweitzer.
\newblock An agent-based model of multi-dimensional opinion dynamics and
  opinion alignment.
\newblock \emph{Chaos: An Interdisciplinary Journal of Nonlinear Science},
  30\penalty0 (9):\penalty0 093139, 2020.

\bibitem[Gaitonde et~al.(2021)Gaitonde, Kleinberg, and
  Tardos]{gaitonde2021polarization}
Jason Gaitonde, Jon Kleinberg, and {\'E}va Tardos.
\newblock Polarization in geometric opinion dynamics.
\newblock In \emph{Proceedings of the 22nd ACM Conference on Economics and
  Computation}, pages 499--519, 2021.

\bibitem[H{\k{a}}z{\l}a et~al.(2019)H{\k{a}}z{\l}a, Jin, Mossel, and
  Ramnarayan]{hkazla2019geometric}
Jan H{\k{a}}z{\l}a, Yan Jin, Elchanan Mossel, and Govind Ramnarayan.
\newblock A geometric model of opinion polarization.
\newblock \emph{arXiv preprint arXiv:1910.05274}, 2019.

\bibitem[Liu et~al.(2018)Liu, Wu, and Lin]{liu2018mean}
Zhiyu Liu, Bo~Wu, and Hai Lin.
\newblock A mean field game approach to swarming robots control.
\newblock In \emph{2018 Annual American Control Conference (ACC)}, pages
  4293--4298. IEEE, 2018.

\bibitem[Elamvazhuthi and Berman(2019)]{elamvazhuthi2019mean}
Karthik Elamvazhuthi and Spring Berman.
\newblock Mean-field models in swarm robotics: A survey.
\newblock \emph{Bioinspiration \& Biomimetics}, 15\penalty0 (1):\penalty0
  015001, 2019.

\bibitem[Lu et~al.(2020)Lu, Ma, Lu, Lu, and Ying]{lu2020mean}
Yiping Lu, Chao Ma, Yulong Lu, Jianfeng Lu, and Lexing Ying.
\newblock A mean-field analysis of deep resnet and beyond: Towards provable
  optimization via overparameterization from depth.
\newblock \emph{arXiv preprint arXiv:2003.05508}, 2020.

\bibitem[Hu et~al.(2019)Hu, Kazeykina, and Ren]{hu2019mean}
Kaitong Hu, Anna Kazeykina, and Zhenjie Ren.
\newblock Mean-field langevin system, optimal control and deep neural networks.
\newblock \emph{arXiv preprint arXiv:1909.07278}, 2019.

\bibitem[Weinan et~al.(2018)Weinan, Han, and Li]{han2018mean}
E~Weinan, Jiequn Han, and Qianxiao Li.
\newblock A mean-field optimal control formulation of deep learning.
\newblock \emph{arXiv preprint arXiv:1807.01083}, 2018.

\bibitem[Lasry and Lions(2007)]{lasry2007mean}
Jean-Michel Lasry and Pierre-Louis Lions.
\newblock Mean field games.
\newblock \emph{Japanese journal of mathematics}, 2\penalty0 (1):\penalty0
  229--260, 2007.

\bibitem[Gu{\'e}ant et~al.(2011)Gu{\'e}ant, Lasry, and Lions]{gueant2011mean}
Olivier Gu{\'e}ant, Jean-Michel Lasry, and Pierre-Louis Lions.
\newblock Mean field games and applications.
\newblock In \emph{Paris-Princeton lectures on mathematical finance 2010},
  pages 205--266. Springer, 2011.

\bibitem[Bensoussan et~al.(2013)Bensoussan, Frehse, Yam,
  et~al.]{bensoussan2013mean}
Alain Bensoussan, Jens Frehse, Phillip Yam, et~al.
\newblock \emph{Mean field games and mean field type control theory}, volume
  101.
\newblock Springer, 2013.

\bibitem[Ruthotto et~al.(2020)Ruthotto, Osher, Li, Nurbekyan, and
  Fung]{ruthotto2020machine}
Lars Ruthotto, Stanley~J Osher, Wuchen Li, Levon Nurbekyan, and Samy~Wu Fung.
\newblock A machine learning framework for solving high-dimensional mean field
  game and mean field control problems.
\newblock \emph{Proceedings of the National Academy of Sciences}, 117\penalty0
  (17):\penalty0 9183--9193, 2020.

\bibitem[Lin et~al.(2021)Lin, Fung, Li, Nurbekyan, and
  Osher]{lin2021alternating}
Alex~Tong Lin, Samy~Wu Fung, Wuchen Li, Levon Nurbekyan, and Stanley~J Osher.
\newblock Alternating the population and control neural networks to solve
  high-dimensional stochastic mean-field games.
\newblock \emph{Proceedings of the National Academy of Sciences}, 118\penalty0
  (31), 2021.

\bibitem[Chen(2021)]{chen2021density}
Yongxin Chen.
\newblock Density control of interacting agent systems.
\newblock \emph{arXiv preprint arXiv:2108.07342}, 2021.

\bibitem[Han et~al.(2018)Han, Jentzen, and Weinan]{han2018solving}
Jiequn Han, Arnulf Jentzen, and E~Weinan.
\newblock Solving high-dimensional partial differential equations using deep
  learning.
\newblock \emph{Proceedings of the National Academy of Sciences}, 115\penalty0
  (34):\penalty0 8505--8510, 2018.

\bibitem[Exarchos and Theodorou(2018)]{exarchos2018stochastic}
Ioannis Exarchos and Evangelos~A Theodorou.
\newblock Stochastic optimal control via forward and backward stochastic
  differential equations and importance sampling.
\newblock \emph{Automatica}, 87:\penalty0 159--165, 2018.

\bibitem[Pereira et~al.(2019)Pereira, Wang, Exarchos, and
  Theodorou]{pereira2019neural}
Marcus Pereira, Ziyi Wang, Ioannis Exarchos, and Evangelos~A Theodorou.
\newblock Neural network architectures for stochastic control using the
  nonlinear feynman-kac lemma.
\newblock \emph{arXiv preprint arXiv:1902.03986}, 2019.

\bibitem[Carmona et~al.(2013)Carmona, Delarue, and
  Lachapelle]{carmona2013control}
Ren{\'e} Carmona, Fran{\c{c}}ois Delarue, and Aim{\'e} Lachapelle.
\newblock Control of mckean--vlasov dynamics versus mean field games.
\newblock \emph{Mathematics and Financial Economics}, 7\penalty0 (2):\penalty0
  131--166, 2013.

\bibitem[Carmona and Delarue(2013)]{carmona2013mean}
Ren{\'e} Carmona and Fran{\c{c}}ois Delarue.
\newblock Mean field forward-backward stochastic differential equations.
\newblock \emph{Electronic Communications in Probability}, 18:\penalty0 1--15,
  2013.

\bibitem[Carmona and Lauri{\`e}re(2019)]{carmona2019convergence}
Ren{\'e} Carmona and Mathieu Lauri{\`e}re.
\newblock Convergence analysis of machine learning algorithms for the numerical
  solution of mean field control and games: Ii--the finite horizon case.
\newblock \emph{arXiv preprint arXiv:1908.01613}, 2019.

\bibitem[Carmona and Lauri{\`e}re(2021)]{carmona2021convergence}
Ren{\'e} Carmona and Mathieu Lauri{\`e}re.
\newblock Convergence analysis of machine learning algorithms for the numerical
  solution of mean field control and games i: the ergodic case.
\newblock \emph{SIAM Journal on Numerical Analysis}, 59\penalty0 (3):\penalty0
  1455--1485, 2021.

\bibitem[De~Bortoli et~al.(2021)De~Bortoli, Thornton, Heng, and
  Doucet]{de2021diffusion}
Valentin De~Bortoli, James Thornton, Jeremy Heng, and Arnaud Doucet.
\newblock Diffusion schr{\"o}dinger bridge with applications to score-based
  generative modeling.
\newblock \emph{arXiv preprint arXiv:2106.01357}, 2021.

\bibitem[Vargas et~al.(2021)Vargas, Thodoroff, Lawrence, and
  Lamacraft]{vargas2021solving}
Francisco Vargas, Pierre Thodoroff, Neil~D Lawrence, and Austen Lamacraft.
\newblock Solving schr{\"o}dinger bridges via maximum likelihood.
\newblock \emph{arXiv preprint arXiv:2106.02081}, 2021.

\bibitem[Wang et~al.(2021)Wang, Jiao, Xu, Wang, and Yang]{wang2021deep}
Gefei Wang, Yuling Jiao, Qian Xu, Yang Wang, and Can Yang.
\newblock Deep generative learning via schr{\"o}dinger bridge.
\newblock \emph{arXiv preprint arXiv:2106.10410}, 2021.

\bibitem[Chen et~al.(2021)Chen, Liu, and Theodorou]{chen2021likelihood}
Tianrong Chen, Guan-Horng Liu, and Evangelos~A Theodorou.
\newblock Likelihood training of schr\"odinger bridge using forward-backward
  sdes theory.
\newblock \emph{arXiv preprint arXiv:2110.11291}, 2021.

\bibitem[Bunne et~al.(2022)Bunne, Hsieh, Cuturi, and
  Krause]{bunne2022recovering}
Charlotte Bunne, Ya-Ping Hsieh, Marco Cuturi, and Andreas Krause.
\newblock Recovering stochastic dynamics via gaussian schr$\backslash$" odinger
  bridges.
\newblock \emph{arXiv preprint arXiv:2202.05722}, 2022.

\bibitem[Schr{\"o}dinger(1932)]{schrodinger1932theorie}
Erwin Schr{\"o}dinger.
\newblock Sur la th{\'e}orie relativiste de l'{\'e}lectron et
  l'interpr{\'e}tation de la m{\'e}canique quantique.
\newblock In \emph{Annales de l'institut Henri Poincar{\'e}}, volume~2, pages
  269--310, 1932.

\bibitem[Caluya and Halder(2021)]{caluya2021wasserstein}
Kenneth Caluya and Abhishek Halder.
\newblock Wasserstein proximal algorithms for the schr{\"o}dinger bridge
  problem: Density control with nonlinear drift.
\newblock \emph{IEEE Transactions on Automatic Control}, 2021.

\bibitem[Backhoff et~al.(2020)Backhoff, Conforti, Gentil, and
  L{\'e}onard]{backhoff2020mean}
Julio Backhoff, Giovanni Conforti, Ivan Gentil, and Christian L{\'e}onard.
\newblock The mean field schr{\"o}dinger problem: ergodic behavior, entropy
  estimates and functional inequalities.
\newblock \emph{Probability Theory and Related Fields}, 178\penalty0
  (1):\penalty0 475--530, 2020.

\bibitem[Chen et~al.(2015)Chen, Georgiou, and Pavon]{chen2015optimal}
Yongxin Chen, Tryphon Georgiou, and Michele Pavon.
\newblock Optimal steering of inertial particles diffusing anisotropically with
  losses.
\newblock In \emph{2015 American Control Conference (ACC)}, pages 1252--1257.
  IEEE, 2015.

\bibitem[Schr{\"o}dinger(1931)]{schrodinger1931umkehrung}
Erwin Schr{\"o}dinger.
\newblock \emph{{\"U}ber die umkehrung der naturgesetze}.
\newblock Verlag der Akademie der Wissenschaften in Kommission bei Walter De
  Gruyter u~…, 1931.

\bibitem[L{\'e}onard(2012)]{leonard2012schrodinger}
Christian L{\'e}onard.
\newblock From the schr{\"o}dinger problem to the monge--kantorovich problem.
\newblock \emph{Journal of Functional Analysis}, 262\penalty0 (4):\penalty0
  1879--1920, 2012.

\bibitem[L{\'e}onard(2013)]{leonard2013survey}
Christian L{\'e}onard.
\newblock A survey of the schr$\backslash$" odinger problem and some of its
  connections with optimal transport.
\newblock \emph{arXiv preprint arXiv:1308.0215}, 2013.

\bibitem[Pavon and Wakolbinger(1991)]{pavon1991free}
Michele Pavon and Anton Wakolbinger.
\newblock On free energy, stochastic control, and schr{\"o}dinger processes.
\newblock In \emph{Modeling, Estimation and Control of Systems with
  Uncertainty}, pages 334--348. Springer, 1991.

\bibitem[Dai~Pra(1991)]{dai1991stochastic}
Paolo Dai~Pra.
\newblock A stochastic control approach to reciprocal diffusion processes.
\newblock \emph{Applied mathematics and Optimization}, 23\penalty0
  (1):\penalty0 313--329, 1991.

\bibitem[Vargas(2021)]{vargas2021machine}
Francisco Vargas.
\newblock Machine-learning approaches for the empirical schr{\"o}dinger bridge
  problem.
\newblock Technical report, University of Cambridge, Computer Laboratory, 2021.

\bibitem[Benamou et~al.(2015)Benamou, Carlier, Cuturi, Nenna, and
  Peyr{\'e}]{benamou2015iterative}
Jean-David Benamou, Guillaume Carlier, Marco Cuturi, Luca Nenna, and Gabriel
  Peyr{\'e}.
\newblock Iterative bregman projections for regularized transportation
  problems.
\newblock \emph{SIAM Journal on Scientific Computing}, 37\penalty0
  (2):\penalty0 A1111--A1138, 2015.

\bibitem[Kullback(1968)]{kullback1968probability}
Solomon Kullback.
\newblock Probability densities with given marginals.
\newblock \emph{The Annals of Mathematical Statistics}, 39\penalty0
  (4):\penalty0 1236--1243, 1968.

\bibitem[Zhang and Chen(2021)]{zhang2021path}
Qinsheng Zhang and Yongxin Chen.
\newblock Path integral sampler: a stochastic control approach for sampling.
\newblock \emph{arXiv preprint arXiv:2111.15141}, 2021.

\bibitem[Hopf(1950)]{hopf1950partial}
Eberhard Hopf.
\newblock The partial differential equation ut+ uux= $\mu$xx.
\newblock \emph{Communications on Pure and Applied mathematics}, 3\penalty0
  (3):\penalty0 201--230, 1950.

\bibitem[Cole(1951)]{cole1951quasi}
Julian~D Cole.
\newblock On a quasi-linear parabolic equation occurring in aerodynamics.
\newblock \emph{Quarterly of applied mathematics}, 9\penalty0 (3):\penalty0
  225--236, 1951.

\bibitem[Yong and Zhou(1999)]{yong1999stochastic}
Jiongmin Yong and Xun~Yu Zhou.
\newblock \emph{Stochastic controls: Hamiltonian systems and HJB equations},
  volume~43.
\newblock Springer Science \& Business Media, 1999.

\bibitem[Kobylanski(2000)]{kobylanski2000backward}
Magdalena Kobylanski.
\newblock Backward stochastic differential equations and partial differential
  equations with quadratic growth.
\newblock \emph{Annals of probability}, pages 558--602, 2000.

\bibitem[It{\^o}(1951)]{ito1951stochastic}
Kiyosi It{\^o}.
\newblock \emph{On stochastic differential equations}, volume~4.
\newblock American Mathematical Soc., 1951.

\bibitem[Bellman(1954)]{bellman1954theory}
Richard Bellman.
\newblock The theory of dynamic programming.
\newblock Technical report, Rand corp santa monica ca, 1954.

\bibitem[Todorov(2009)]{todorov2009efficient}
Emanuel Todorov.
\newblock Efficient computation of optimal actions.
\newblock \emph{Proceedings of the national academy of sciences}, 106\penalty0
  (28):\penalty0 11478--11483, 2009.

\bibitem[Lutter et~al.(2021)Lutter, Mannor, Peters, Fox, and
  Garg]{lutter2021value}
Michael Lutter, Shie Mannor, Jan Peters, Dieter Fox, and Animesh Garg.
\newblock Value iteration in continuous actions, states and time.
\newblock \emph{arXiv preprint arXiv:2105.04682}, 2021.

\bibitem[Meng et~al.(2021)Meng, Gorbet, and Kuli{\'c}]{meng2021effect}
Lingheng Meng, Rob Gorbet, and Dana Kuli{\'c}.
\newblock The effect of multi-step methods on overestimation in deep
  reinforcement learning.
\newblock In \emph{2020 25th International Conference on Pattern Recognition
  (ICPR)}, pages 347--353. IEEE, 2021.

\bibitem[van Seijen(2016)]{van2016effective}
Harm van Seijen.
\newblock Effective multi-step temporal-difference learning for non-linear
  function approximation.
\newblock \emph{arXiv preprint arXiv:1608.05151}, 2016.

\bibitem[Hessel et~al.(2018)Hessel, Modayil, Van~Hasselt, Schaul, Ostrovski,
  Dabney, Horgan, Piot, Azar, and Silver]{hessel2018rainbow}
Matteo Hessel, Joseph Modayil, Hado Van~Hasselt, Tom Schaul, Georg Ostrovski,
  Will Dabney, Dan Horgan, Bilal Piot, Mohammad Azar, and David Silver.
\newblock Rainbow: Combining improvements in deep reinforcement learning.
\newblock In \emph{Thirty-second AAAI conference on artificial intelligence},
  2018.

\bibitem[Mnih et~al.(2013)Mnih, Kavukcuoglu, Silver, Graves, Antonoglou,
  Wierstra, and Riedmiller]{mnih2013playing}
Volodymyr Mnih, Koray Kavukcuoglu, David Silver, Alex Graves, Ioannis
  Antonoglou, Daan Wierstra, and Martin Riedmiller.
\newblock Playing atari with deep reinforcement learning.
\newblock \emph{arXiv preprint arXiv:1312.5602}, 2013.

\bibitem[Lillicrap et~al.(2015)Lillicrap, Hunt, Pritzel, Heess, Erez, Tassa,
  Silver, and Wierstra]{lillicrap2015continuous}
Timothy~P Lillicrap, Jonathan~J Hunt, Alexander Pritzel, Nicolas Heess, Tom
  Erez, Yuval Tassa, David Silver, and Daan Wierstra.
\newblock Continuous control with deep reinforcement learning.
\newblock \emph{arXiv preprint arXiv:1509.02971}, 2015.

\bibitem[Schulman et~al.(2015)Schulman, Levine, Abbeel, Jordan, and
  Moritz]{schulman2015trust}
John Schulman, Sergey Levine, Pieter Abbeel, Michael Jordan, and Philipp
  Moritz.
\newblock Trust region policy optimization.
\newblock In \emph{International conference on machine learning}, pages
  1889--1897, 2015.

\bibitem[Loshchilov and Hutter(2017)]{loshchilov2017decoupled}
Ilya Loshchilov and Frank Hutter.
\newblock Decoupled weight decay regularization.
\newblock \emph{arXiv preprint arXiv:1711.05101}, 2017.

\bibitem[Lord et~al.(1979)Lord, Ross, and Lepper]{lord1979biased}
Charles~G Lord, Lee Ross, and Mark~R Lepper.
\newblock Biased assimilation and attitude polarization: The effects of prior
  theories on subsequently considered evidence.
\newblock \emph{Journal of personality and social psychology}, 37\penalty0
  (11):\penalty0 2098, 1979.

\bibitem[Dandekar et~al.(2013)Dandekar, Goel, and Lee]{dandekar2013biased}
Pranav Dandekar, Ashish Goel, and David~T Lee.
\newblock Biased assimilation, homophily, and the dynamics of polarization.
\newblock \emph{Proceedings of the National Academy of Sciences}, 110\penalty0
  (15):\penalty0 5791--5796, 2013.

\bibitem[Pardoux and Peng(1992)]{pardoux1992backward}
Etienne Pardoux and Shige Peng.
\newblock Backward stochastic differential equations and quasilinear parabolic
  partial differential equations.
\newblock In \emph{Stochastic partial differential equations and their
  applications}, pages 200--217. Springer, 1992.

\bibitem[Negyesi et~al.(2021)Negyesi, Andersson, and Oosterlee]{negyesi2021one}
Balint Negyesi, Kristoffer Andersson, and Cornelis~W Oosterlee.
\newblock The one step malliavin scheme: new discretization of bsdes
  implemented with deep learning regressions.
\newblock \emph{arXiv preprint arXiv:2110.05421}, 2021.

\bibitem[Li and Hao(2018)]{li2018optimal}
Qianxiao Li and Shuji Hao.
\newblock An optimal control approach to deep learning and applications to
  discrete-weight neural networks.
\newblock In \emph{International Conference on Machine Learning}, pages
  2985--2994. PMLR, 2018.

\bibitem[Liu et~al.(2021{\natexlab{a}})Liu, Chen, and
  Theodorou]{liu2021differential}
Guan-Horng Liu, Tianrong Chen, and Evangelos~A Theodorou.
\newblock Ddpnopt: Differential dynamic programming neural optimizer.
\newblock In \emph{International Conference on Learning Representations},
  2021{\natexlab{a}}.

\bibitem[Liu et~al.(2021{\natexlab{b}})Liu, Chen, and Theodorou]{liu2021second}
Guan-Horng Liu, Tianrong Chen, and Evangelos~A Theodorou.
\newblock Second-order neural ode optimizer.
\newblock In \emph{Advances in Neural Information Processing Systems},
  2021{\natexlab{b}}.

\bibitem[Song et~al.(2020)Song, Sohl-Dickstein, Kingma, Kumar, Ermon, and
  Poole]{song2020score}
Yang Song, Jascha Sohl-Dickstein, Diederik~P Kingma, Abhishek Kumar, Stefano
  Ermon, and Ben Poole.
\newblock Score-based generative modeling through stochastic differential
  equations.
\newblock \emph{arXiv preprint arXiv:2011.13456}, 2020.

\bibitem[Nelson(2020)]{nelson2020dynamical}
Edward Nelson.
\newblock \emph{Dynamical theories of Brownian motion}, volume 106.
\newblock Princeton university press, 2020.

\bibitem[Anderson(1982)]{anderson1982reverse}
Brian~DO Anderson.
\newblock Reverse-time diffusion equation models.
\newblock \emph{Stochastic Processes and their Applications}, 12\penalty0
  (3):\penalty0 313--326, 1982.

\bibitem[Pavliotis(2014)]{pavliotis2014stochastic}
Grigorios~A Pavliotis.
\newblock \emph{Stochastic processes and applications: diffusion processes, the
  Fokker-Planck and Langevin equations}, volume~60.
\newblock Springer, 2014.

\bibitem[Song et~al.(2021)Song, Durkan, Murray, and Ermon]{song2021maximum}
Yang Song, Conor Durkan, Iain Murray, and Stefano Ermon.
\newblock Maximum likelihood training of score-based diffusion models.
\newblock \emph{arXiv e-prints}, pages arXiv--2101, 2021.

\bibitem[Huang et~al.(2021)Huang, Lim, and Courville]{huang2021variational}
Chin-Wei Huang, Jae~Hyun Lim, and Aaron Courville.
\newblock A variational perspective on diffusion-based generative models and
  score matching.
\newblock \emph{arXiv preprint arXiv:2106.02808}, 2021.

\bibitem[Elfwing et~al.(2018)Elfwing, Uchibe, and Doya]{elfwing2018sigmoid}
Stefan Elfwing, Eiji Uchibe, and Kenji Doya.
\newblock Sigmoid-weighted linear units for neural network function
  approximation in reinforcement learning.
\newblock \emph{Neural Networks}, 107:\penalty0 3--11, 2018.

\bibitem[Paszke et~al.(2017)Paszke, Gross, Chintala, Chanan, Yang, DeVito, Lin,
  Desmaison, Antiga, and Lerer]{paszke2017automatic}
Adam Paszke, Sam Gross, Soumith Chintala, Gregory Chanan, Edward Yang, Zachary
  DeVito, Zeming Lin, Alban Desmaison, Luca Antiga, and Adam Lerer.
\newblock Automatic differentiation in pytorch.
\newblock 2017.

\end{thebibliography}

\newpage
\appendix

\section{Appendix}

\noindent\rule{\textwidth}{0.5pt}
\input{subtex/appendix.toc}
\noindent\rule{\textwidth}{0.5pt}

\subsection{Summary of Abbreviation and Notation} \label{sec:a1}

\begin{figure}[H]
\vskip -0.1in
\begin{minipage}[t]{0.55\textwidth}
    \centering
    \captionsetup{type=table}
    \captionsetup{justification=centering}
    \caption{
        Abbreviation.
    }
    \label{table:abbrev}
    \setlength\tabcolsep{4pt}
    \centering
    \begin{tabular}{ll}
        \toprule
        MFGs  & Mean-Field Games \\[3pt]
        SB & \SB \\[3pt]
        DeepRL & Deep Reinforcement Learning \\[3pt]
        PDEs & Partial Differential Equations \\[3pt]
        HJB & Hamilton-Jacobi-Bellman \\[3pt]
        FP &  Fokker-Plank \\[3pt]
        SDEs & Stochastic Differential Equations\\[3pt]
        FBSDEs & Forward-Backward SDEs \\[3pt]
        IPF & Iterative Proportional Fitting  \\[3pt]
        MF interaction & Mean-field interaction \\[3pt]
        nonlinear FK & nonlinear Feynman-Kac \\[3pt]
        TD & Temporal Difference \\[3pt]
        \bottomrule
    \end{tabular}
\end{minipage}
\hfill
\begin{minipage}[t]{0.42\textwidth}
    \centering
    \captionsetup{type=table}
    \captionsetup{justification=centering}
    \caption{
        Notation.
    }
    \label{table:notation}
    \setlength\tabcolsep{4pt}
    \centering
    \begin{tabular}{ll}
        \toprule
        $t$  & time coordinate \\[3pt]
        $s$  & reversed time coordinate \\[3pt]
        $u(t,x)$ & value function \\[3pt]
        $\rho(t,x)$ & marginal distribution \\[3pt]
        $\rho_0,\target$ & initial/target distributions \\[3pt]
        $H$ & Hamiltonian function \\[3pt]
        $F$ & MF interaction function \\[3pt]
        $f$ & MF base drift \\[3pt]
        $\sigma$ & diffusion scaler \\[3pt]
        $(\Psi, \Psihat)$ & solution to SB PDEs \\[3pt]
        $(Y, Z)$ & nonlinear FK of $\Psi$ \\[3pt]
        $(\Yhat, \Zhat)$ & nonlinear FK of $\Psihat$  \\[3pt]
        $\TD_s$ & TD target for $Y_s$ \\[3pt]
        $\TDhat_t$ & TD target for $\Yhat_t$ \\[3pt]
        $\theta$ & Parameter of $Y$ (and $Z$) \\[3pt]
        $\phi$ & Parameter of $\Yhat$ (and $\Zhat$) \\[3pt]
        \bottomrule
    \end{tabular}
\end{minipage}
\end{figure}

\subsection{Review of Nonlinear FK Lemma and SB-FBSDE} \label{sec:a2}

\begin{lemma}[Nonlinear Feynman-Kac Lemma \citep{yong1999stochastic,kobylanski2000backward,exarchos2018stochastic}]\label{lemma:non-fc}%
    Let $v\equiv v(x,t)$ be a function that is twice continuously differentiable in $x \in \sR^d$ and once differentiable in $t \in [0,T]$, \ie $v \in C^{2,1}(\sR^d,[0,T])$.
    Consider the following second-order parabolic PDE,
    \begin{align}
      \fracpartial{v}{t} +
      \frac{1}{2}\Tr(\nabla^2 v~G(x,t)G(x,t)^\T) + \nabla v^\T f(x,t) + h(x,v,G(x,t)^\T \nabla v, t)  = 0,
      \text{ }\text{ } v(T,x) = \varphi(x),
      \label{eq:nfk-pde}
    \end{align}
    where $\nabla^2$ denotes the Hessian operator w.r.t. $x$ and
    the functions $f$, $G$, $h$, and ${\varphi}$ satisfy proper regularity conditions.
    Specifically,
    \textit{(i)} $f$, $G$, $h$, and ${\varphi}$ are continuous,
    \textit{(ii)} $f(x,t)$ and $G(x,t)$ are uniformly Lipschitz in $x$, and
    \textit{(iii)} $h(x,y,z,t)$ satisfies quadratic growth condition in $z$.
    Then, \eqref{eq:nfk-pde} exists a unique solution $v$ such that
    the following stochastic representation (known as the nonlinear Feynman-Kac transformation) holds:
    \begin{align}
      Y_t = v(X_t,t),
      \qquad\qquad
      Z_t = G(X_t,t)^\T \nabla v(X_t,t),
      \label{eq:nfk}
    \end{align}
    where $(X_t, Y_t, Z_t)$ are the unique adapted solutions to the following FBSDEs:
    \begin{equation}
        \begin{split}
            \rd X_t &=  f(X_t,t) \dt + G(X_t,t) \rd W_t, \quad X_0 = x_0,
            \\
            \rd Y_t &=  - h(X_t, Y_t, Z_t,t) \dt + Z_t^\T \rd W_t, \quad Y_T = \varphi(X_T).
            \label{eq:nfk-fbsde}
        \end{split}
    \end{equation}
  The original deterministic PDE solution $v(x,t)$ can be recovered by taking conditional expectations:
  \begin{align*}
    \E\br{Y_t | X_t = x} = v(x,t),
    \qquad\qquad
    \E\br{Z_t | X_t = x} = G(x,t)^\T \nabla v(x,t).
  \end{align*}
\end{lemma}

Lemma \ref{lemma:non-fc} establishes an intriguing connection between
a certain class of (nonlinear) PDEs in \eqref{eq:nfk-pde} and FBSDEs \eqref{eq:nfk-fbsde} via the
nonlinear FK transformation \eqref{eq:nfk}.
In this work, we adopt a simpler diffusion $G(x,t) := \sigma$ as a time-invariant scalar
but note that our derivation can be extended to more general cases straightforwardly.

\textbf{Viscosity solution. $\quad$}
Lemma \ref{lemma:non-fc} can be extended to viscosity solutions when the classical solution does not exist.
In which case, we will have $v(x, t) = \lim_{\epsilon\rightarrow\infty} v^{\epsilon}(x, t)$ converge uniformly in $(x, t)$ over a compact set, where $v^{\epsilon}(x, t)$ is the classical solution to \eqref{eq:nfk-pde} with $(f_{\epsilon},G_{\epsilon},h_{\epsilon},\varphi_{\epsilon})$ converge uniformly toward $(f,G,h,\varphi)$ over the compact set; see \citep{pardoux1992backward,exarchos2018stochastic,negyesi2021one} for a complete discussion.

\textbf{SB-FBSDE \citep{chen2021likelihood}. $\quad$}
SB-FBSDE is a new class of generative models that, inspiring by the recent advance of understanding deep learning through the optimal control perspective \citep{li2018optimal,liu2021differential,liu2021second},
adopts Lemma \ref{lemma:non-fc} to generalize the score-based diffusion models.
Since the PDEs $(\fracpartial{\Psi}{t}, \fracpartial{\Psihat}{t})$ appearing in the vanilla SB \eqref{eq:sb-pde} are both of the parabolic form \eqref{eq:nfk-pde},
one can apply Lemma \ref{lemma:non-fc} and derive the corresponding nonlinear generators $h$.
This, as shown in SB-FBSDE \citep{chen2021likelihood}, leads to the following FBSDEs:
\begin{subequations}
    \label{eq:prior}
    \begin{empheq}[left={\empheqlbrace}]{align}
        \rd X_t &= \pr{f_t + \sigma Z_t} \dt + \sigma \rd W_t \label{eq:fsde-old} \\
        \rd Y_t &= {\frac{1}{2} \norm{Z_t}^2 } \dt + Z_t^\T \rd W_t \label{eq:bsde-old}
        \\
        \rd \Yhat_t &= \pr{\frac{1}{2} \norm{\Zhat_t}^2 +
                        \nabla \cdot (\sigma \Zhat_t -f_t) + \Zhat_t^\T Z_t
                    } \dt + \Zhat_t^\T \rd W_t \label{eq:bsde2-old}
    \end{empheq}
\end{subequations}
Further, the nonlinear FK transformation reads
\begin{equation}
    \begin{alignedat}{2}
      Y_t &= \log \Psi(X_t, t), \qquad
      Z_t &&= \sigma~\nabla \log \Psi(X_t, t), \\
      \Yhat_t &=          \log \Psihat(X_t, t), \qquad
      \Zhat_t &&= \sigma~\nabla \log \Psihat(X_t, t), \nonumber
    \end{alignedat}
\end{equation}
which immediately suggests that
\begin{align}
    \E[Y_t | X_t = x] = \log \Psi(x,t), \qquad
    \E[\Yhat_t | X_t = x] = \log \Psihat(x,t).
    \label{eq:nfk2}
\end{align}
It can be readily seen that \eqref{eq:prior} is a special case of our Theorem~\ref{thm:mfg-fbsde}
when the MF interaction $F(x,\rho)$, which plays a crucial role in MFGs, vanishes.
Since SB-FBSDE was primarily developed in the context of generative modeling \citep{song2020score},
its training relies on computing the log-likelihood at the boundaries.
These log-likelihoods can be obtained by noticing that
$\log \rho(x,t) = \E[Y_t + \Yhat_t | X_t = x]$, as implied by \eqref{eq:nfk2} and \eqref{eq:hf}.
When $\Zhat_\phi(X_t,t) \approx \Zhat_t$ and $Z_\theta(\Xbar_s,s) \approx Z_s$,
the training objectives of SB-FBSDE can be computed as the parametrized variational lower-bounds:
\begin{subequations}
    \label{eq:nll}
    \begin{align}
        \log \rho_0(\phi; x)   &\ge \Lipf(\phi)
            := \E[Y_t^\theta + \Yhat_t^\phi | X_t = x, t=0]
            = \int_t \E \br{ \rd Y_t^\theta + \rd \Yhat_t^\phi | X_0 = x}, \label{eq:nll0} \\
        \log \rho_T(\theta; x) &\ge \Lipf(\theta)
            := \E[Y_s^\theta + \Yhat_s^\phi | \Xbar_s = x, s=0]
            = \int_s \E \br{ \rd Y_s^\theta + \rd \Yhat_s^\phi | \Xbar_0 = x}. \label{eq:nllT}
    \end{align}
\end{subequations}
Invoking \eqref{eq:prior} to expand the \textit{r.h.s.} of \eqref{eq:nll} leads to the expression in \eqref{eq:L-ipf}:
    \begin{align*}
        \Lipf(\theta) &= \int_0^T \E_{\text{\eqref{eq:rsde}}}\br{
            \frac{1}{2} \norm{Z_\theta(\Xbar_s, s)}_2^2 +  Z_\theta(\Xbar_s, s)^\T \Zhat_\phi(\Xbar_s, s) + \nabla\cdot (\sigma Z_\theta(\Xbar_s, s) {+} f)
        } \ds, \\
        \Lipf(\phi) &= \int_0^T \E_{\text{\eqref{eq:sde}}}\br{
            \frac{1}{2} \norm{\Zhat_\phi(X_t, t)}_2^2 + \Zhat_\phi(X_t, t)^\T Z_\theta(X_t,t) + \nabla\cdot(\sigma\Zhat_\phi(X_t, t) {-} f)
        } \dt.
    \end{align*}
Since \eqref{eq:nll} concern only the integration over the expectations, \ie $\int \E [\rd Y + \rd \Yhat]$,
the solutions $(Y_t, \Yhat_t)$ to the SDEs (\ref{eq:bsde-old}, \ref{eq:bsde2-old}) were \textit{never} computed explicitly in SB-FBSDE,
This is in contrast to our DeepGSB, which, crucially, requires computing $(Y_t, \Yhat_t)$ explicitly and regress their values with TD objectives, so that the stochastic dynamics of $\rd Y$ and $\rd\Yhat$ are respectively respected.

\def\expu{{ \exp\left( -u \right) }}
\def\expv{{ \exp\left( u \right) }}
\def\Du{{ \nabla u }}
\def\Dm{{ \nabla \rho }}

\subsection{Proofs in Main Paper} \label{sec:a3}

Throughout this section, we will denote the parameterized forward and backward SDEs by
\begin{subequations}
    \label{eqq:25}
    \begin{align}
    \rd X^\theta_t &= \pr{f_t + \sigma Z_\theta(X^\theta_t,t) } \dt + \sigma \rd W_t, \label{eqq:j} \\
    \rd \Xbar^\phi_s &= \pr{-f_s + \sigma \Zhat_\phi(\Xbar^\phi_s,t) } \ds + \sigma \rd W_s, \label{eqq:k}
    \end{align}
\end{subequations}
and denote their time-marginal densities respectively as $q^\theta$ and $q^\phi$.

\subsubsection{Preliminary} \label{sec:a3.1}

We first restate some useful lemmas that will appear in the proceeding proofs.

\begin{lemma}[It{\^o} formula \citep{ito1951stochastic}] \label{lemma:ito}
    Let $X_t$ be the solution to the It\^o SDE:
    \begin{align*}
    \rd X_t = f(X_t, t) \dt + \sigma(X_t, t) \rd W_t.
    \end{align*}
    Then, the stochastic process $v(X_t,t)$, where $v \in C^{2,1}(\sR^d,[0,T])$, is also an It{\^o} process satisfying
    \begin{align*}
    \rd v(X_t,t) =
        \fracpartial{v(X_t,t)}{t}\dt
        &+ \br{ \nabla v(X_t, t)^\T f + \frac{1}{2}\Tr\br{\sigma^\T \nabla^2v(X_t, t) \sigma} } \dt
        + \br{ \nabla v(X_t, t)^\T \sigma } \rd W_t.
        \numberthis \label{eq:ito}
    \end{align*}
\end{lemma}

\begin{lemma} \label{lemma:sbp}
    The following equality holds at any point $x \in \mathbb{R}^n$ such that $p(x)\neq 0$.
    \begin{align*}
    \frac{1}{p(x)} \Delta p(x) = \norm{\nabla \log p(x)}^2 + \Delta \log p(x)
    \end{align*}
\end{lemma}
\begin{proof}
    $ \frac{1}{p(x)} \Delta p(x) = \frac{1}{p(x)} \nabla \cdot \nabla p(x) = \frac{1}{p(x)} \nabla \cdot \pr{ p(x) \nabla \log p(x) }$. Applying chain rule to the divergence yields the desired result.
\end{proof}

\begin{lemma}[\citet{vargas2021machine}, Proposition 1, Sec 6.3.1] \label{lemma:a}
    \begin{align*}
        \rd \log q_t^\phi = \br{ \nabla \cdot \pr{\sigma \Zhat_\phi - f_t}
            + \sigma \pr{Z_\theta + \Zhat_\phi}^\T \nabla \log q_t^\phi
            - \frac{1}{2} \norm{\sigma \nabla \log q_t^\phi}^2
        } \dt + \sigma {\nabla \log q_t^\phi}^\T \rd W_t.
    \end{align*}
\end{lemma}
\begin{proof}
    Invoking Ito lemma w.r.t. the parameterized forward SDE \eqref{eqq:j},
    \begin{align*}
        \rd \log q_t^\phi &= \br{\fracpartial{\log q_t^\phi}{t} + {\nabla \log q_t^\phi}^\T \pr{f_t + \sigma Z_\theta}  + \frac{\sigma^2}{2} \Delta \log q_t^\phi} \dt + \sigma {\nabla \log q_t^\phi}^\T \rd W_t,
    \end{align*}
    where $\fracpartial{\log q_t^\phi}{t}$ obeys (see Eq 13.4 in \citet{nelson2020dynamical}):
    \begin{align*}
        - \fracpartial{q_t^\phi}{t} &= - \nabla \cdot \pr{ \pr{\sigma \Zhat_\phi - f_t} q_t^\phi } + \frac{\sigma^2}{2} \Delta q_t^\phi \\
        \Rightarrow
        \fracpartial{\log q_t^\phi}{t} &= \nabla \cdot \pr{\sigma \Zhat_\phi - f_t} + \pr{\sigma \Zhat_\phi - f_t}^\T \nabla \log q_t^\phi - \frac{\sigma^2\Delta q_t^\phi}{2 q_t^\phi}.
    \end{align*}
    Substituting the above relation yields the desired results.
\end{proof}

\begin{proposition}[\citet{vargas2021machine}, Proposition 1 in Sec 6.3.1] \label{prop:a}
    \begin{align*}
        \KL(q^\theta||q^\phi) &=
        \int_0^T \E_{q^\theta_t} \br{ \frac{1}{2}\norm{ \Zhat_\phi + Z_\theta}^2 +  \nabla \cdot \pr{\sigma \Zhat_\phi - f_t} } \dt + \E_{q^\theta_0} \br{\log \rho_0} - \E_{q^\theta_T} \br{\log \target }
    \end{align*}
\end{proposition}
\begin{proof}
    Recall that the parametrized backward SDE \eqref{eqq:k} can be reversed \citep{anderson1982reverse,song2020score} as
    \begin{align*}
        \rd \Xbar^\phi_t &= \pr{f_t - \sigma \Zhat_\phi(\Xbar^\phi_t,t) + \sigma^2 \nabla \log q^\phi(\Xbar^\phi_t,t) } \dt + \sigma \rd W_t.
    \end{align*}
    Then, we have
    \begin{align*}
        &\quad\quad\KL(q^\theta||q^\phi) \\
        &=
        \int_0^T \E_{q^\theta_t} \br{  \frac{1}{2}\norm{ \Zhat_\phi + Z_\theta - \sigma \nabla \log q_t^\phi}^2 }\dt + \KL(\rho_0||q^\phi_{t=0})\\
        &=
        \int_0^T \E_{q^\theta_t} \br{ \frac{1}{2}\norm{ \Zhat_\phi + Z_\theta}^2  - \sigma (\Zhat_\phi + Z_\theta)^T \nabla \log q_t^\phi + \frac{1}{2}\norm{\sigma\nabla \log q_t^\phi}^2 }\dt + \KL(\rho_0||q^\phi_{t=0}) \\
        \overset{(*)}&{=}~
        \int_0^T \E_{q^\theta_t} \br{ \frac{1}{2}\norm{ \Zhat_\phi + Z_\theta}^2
            + \nabla \cdot \pr{\sigma \Zhat_\phi - f_t}
        }\dt - \E_{q^\theta} \br{ \int_0^T \rd \log q_t^\phi} + \KL(\rho_0||q^\phi_{t=0}) \\
        &=
        \int_0^T \E_{q^\theta_t} \br{ \frac{1}{2}\norm{ \Zhat_\phi + Z_\theta}^2 +  \nabla \cdot \pr{\sigma \Zhat_\phi - f_t} } \dt + \E_{q^\theta_0} \br{\log \rho_0} - \E_{q^\theta_T} \br{\log \target },
    \end{align*}
    where (*) is due to Lemma \ref{lemma:a}.
\end{proof}

\subsubsection{Proof of Lemma~\ref{lemma:kl-sbfbsde}} \label{sec:a3.lemma}
\begin{proof}
    Substituting $\Lipf(\phi)$ into Proposition~\ref{prop:a} and dropping all terms independent of $\phi$ readily yields $\KL(q^\theta || q^\phi) \propto \Lipf(\phi)$. A similar relation can be derived between $\KL(q^\phi || q^\theta)$.
\end{proof}

\textbf{Remark (an alternative simpler proof).}
Suppose $(Z_\theta, q^\theta)$ and $(\Zhat_\phi,q^\phi)$ satisfy proper regularity
such that
$\forall t,s \in [0,T], \quad \exists k>0: q^\theta(x,t) = \calO(\exp^{-\norm{x}_k^2})$, $q^\phi(x,s) = \calO(\exp^{-\norm{x}_k^2})$ as $x \rightarrow \infty$.
Then, an alternative proof using integration by part goes as follows:
Recall that the parametrized forward SDE in \eqref{eqq:j} can be reversed \citep{anderson1982reverse,song2020score} as
\begin{align*}
    \rd X^\theta_s &= \pr{- f_s - \sigma Z_\theta(X^\theta_s,s) + \sigma^2 \nabla \log q^\theta(X^\theta_s,s) } \ds + \sigma \rd W_s.
\end{align*}
Then, the KL divergence can be computed as
\begin{align*}
    &\quad\text{ }\text{ }\KL(q^\theta || q^\phi) \\
    \overset{(*)}&{=}~\E_{q^\theta} \br{\int_0^T \frac{1}{2\sigma^2}\norm{ \sigma \Zhat_\phi + \sigma Z_\theta - \sigma^2 \nabla \log q_s^\theta }^2\ds } + \KL(q^\theta_{s=0}||\target) \numberthis \label{eq:L-ipf22} \\
    &= \int_0^T \E_{q_s^\theta} \br{
        \frac{1}{2}\norm{\Zhat_\phi + Z_\theta}^2 - \sigma (\Zhat_\phi + Z_\theta)^\T \nabla \log q_s^\theta
        + \frac{1}{2}\norm{\sigma \nabla \log q_s^\theta}^2
    } \ds  + \KL(q^\theta_{0}||\target) \\
    &= \int_0^T \E_{q_s^\theta} \br{
        \frac{1}{2}\norm{\Zhat_\phi}^2 + \Zhat_\phi^\T Z_\theta  \markgreen{- \sigma{\Zhat_\phi}^\T \nabla \log q_s^\theta}
    } \ds + \calO(1) \\
    \overset{(**)}&{=}~\int_0^T \E_{q_s^\theta} \br{
        \frac{1}{2}\norm{\Zhat_\phi}^2 + \Zhat_\phi^\T Z_\theta
        \markgreen{+ \sigma\nabla \cdot \Zhat_\phi}
    } \ds + \calO(1),  \\
    &\propto \Lipf(\phi)
\end{align*}
where
(*) is due to the Girsanov’s Theorem \citep{pavliotis2014stochastic} and (**) is due to \markgreen{integration by parts}. $\calO(1)$ collects terms independent of $\phi$.
Notice that the boundary terms vanish due to the additional regularity assumptions on $q^\theta$ and $q^\phi$.
Similar transformations have been adopted in \eg Theorem 1 in \citet{song2021maximum} or Theorem 3 in \citet{huang2021variational}.

\subsubsection{Proof of Theorem~\ref{thm:mfg-fbsde}} \label{sec:a3.thm}
\begin{proof}
    Apply the It\^o formula to $v := \log \Psi(X_t,t)$, where $X_t$ follows \eqref{eq:sde},
    \begin{align*}
        \rd \log \Psi
        = \fracpartial{\log \Psi}{t}\dt
        &+ \br{\nabla \log \Psi^\T (f + \sigma^2 \nabla \log \Psi) + \frac{\sigma^2}{2} \Delta \log \Psi } \dt
        + \sigma\nabla \log \Psi^\T \rd W_t,
    \end{align*}
    and notice that the PDE of $\fracpartial{\log \Psi}{t}$ obeys
    \begin{align*}
    \fracpartial{\log \Psi}{t}
        = \markgreen{\frac{1}{\Psi}} \pr{
            - \nabla \Psi^\T f \markgreen{~-\frac{\sigma^2}{2} \Delta \Psi} + F \Psi
        }
        =  - \nabla \log \Psi^\T f \markgreen{~-\frac{\sigma^2}{2} \norm{\nabla \log \Psi}^2 - \frac{\sigma^2}{2} \Delta\log \Psi} + F.
    \end{align*}
    This yields
    \begin{align}
    \rd \log \Psi
    &= \br{
        \frac{1}{2}\norm{\sigma \nabla \log \Psi}^2 + F} \dt + \sigma \nabla \log \Psi^\T \rd W_t. \label{eq:thm1-Yt}
    \end{align}

    Now, apply the same It\^o formula by instead substituting $v := \log \Psihat(X_t,t)$, where $X_t$ follows \eqref{eq:sde},
    \begin{align*}
        \rd \log \Psihat
        = \fracpartial{\log \Psihat}{t}\dt
        &+ \br{\nabla \log \Psihat^\T (f + \sigma^2 \nabla \log \Psi) + \frac{\sigma^2}{2} \Delta \log \Psihat } \dt
        + \sigma\nabla \log \Psihat^\T \rd W_t,
    \end{align*}
    and notice that the PDE of $\fracpartial{\log \Psihat}{t}$ obeys
    \begin{align*}
        \fracpartial{\log \Psihat}{t}
        &= \markblue{\frac{1}{\Psihat}} \pr{
            - \nabla \cdot (\Psihat f) + \markblue{\frac{\sigma^2}{2} \Delta \Psihat} - F \Psihat
        } \\
        &= - \nabla \log \Psihat^\T f - \nabla \cdot f +
        \markblue{\frac{\sigma^2}{2} \norm{\nabla \log \Psihat}^2 + \frac{\sigma^2}{2} \Delta\log \Psihat}
            - F.
    \end{align*}
    This yields
    \begin{equation}
    \begin{split}
        \rd \log \Psihat
    &= \br{
        - \nabla \cdot f
        + \frac{\sigma^2}{2} \norm{\nabla \log \Psihat}^2 + \sigma^2\nabla \log \Psihat^\T \nabla \log {\Psi} + \sigma^2\Delta\log \Psihat - F
        } \dt + \sigma \nabla \log \Psihat^\T \rd W_t \\
        &= \br{
            \nabla \cdot (\sigma^2 \nabla \log \Psihat - f)
            + \frac{\sigma^2}{2} \norm{\nabla \log \Psihat}^2 + \sigma^2\nabla \log \Psihat^\T \nabla \log {\Psi} - F
        } \dt + \sigma \nabla \log \Psihat^\T \rd W_t.
        \label{eq:thm1-Yhatt}
    \end{split}
    \end{equation}

    Finally, with the nonlinear FK transformation in \eqref{eq:nkc-yz}, \ie
    \begin{equation}
        \begin{alignedat}{2}
        Y_t \equiv Y(X_t, t) &= \log \Psi(X_t, t), \qquad
        Z_t \equiv Z(X_t, t) &&= \sigma~\nabla \log \Psi(X_t, t), \\
        \Yhat_t \equiv \Yhat(X_t, t) &=          \log \Psihat(X_t, t), \qquad
        \Zhat_t \equiv \Zhat(X_t, t) &&= \sigma~\nabla \log \Psihat(X_t, t), \nonumber
        \end{alignedat}
    \end{equation}
    we can rewrite (\ref{eq:sde}, \ref{eq:thm1-Yt}, \ref{eq:thm1-Yhatt}) as the FBSDEs system in \eqref{eq:fbsdet}.
    \begin{subequations}
    \begin{empheq}[box=\widefbox]{align*}
        \rd X_t &= (f_t + \sigma Z_t) \dt + \sigma\rd W_t \\
        \rd Y_t &=  \br{\frac{1}{2} \norm{Z_t}^2 + F_t } \dt + Z_t^\T \rd W_t \\
        \rd \Yhat_t &= \br{
            \frac{1}{2}\norm{\Zhat_t}^2 + \Zhat_t^\T Z_t + \nabla\cdot\left(\sigma\Zhat_t-f_t\right) - F_t
            } + \Zhat^\T \rd W_t
    \end{empheq} %
    \end{subequations}
    where
    \begin{align*}
        f_t := f(X_t, \exp(Y_t+\Yhat_t)), \qquad F_t := F(X_t, \exp(Y_t+\Yhat_t)).
    \end{align*}

    Derivation of the second FBSDEs system in \eqref{eq:fbsdes} follows a similar flow,
    except that we need to rebase the PDEs \eqref{eq:sb2-pde} to the ``\textit{reversed}'' time coordinate $s := T-t$.
    This can be done by reformulating the HJB and FP PDEs in \eqref{eq:mfg2-pde} under the $s$ coordinate,
    then applying the following Hopf-Cole transform:
    \begin{align}
        \Psihat(x,s) := \exp\pr{-u(x,s)}, \quad \Psi(x,s) := \rho(x,s)\exp\pr{u(x,s)}.
    \end{align}
    Notice that we flip the role of $\Psihat(x,s)$ and $\Psi(x,s)$ as the former now relates to the policy
    appearing in \eqref{eq:rsde}.
    Omitting the computation similar to Appendix~\ref{sec:a4.1}, we arrive at the following:
    \begin{align}
        \left\{
        \begin{array}{l}
        \fracpartial{\Psihat(x,s)}{s}
            = \nabla \Psihat^\T f - \frac{1}{2} \sigma^2 \Delta \Psihat + F \Psihat \\[3pt]
        \fracpartial{\Psi(x,s)}{s}
            = \nabla \cdot (\Psi f) + \frac{1}{2} \sigma^2 \Delta \Psi - F \Psi
        \end{array}
        \right.\text{s.t.}
        \begin{array}{l}
            \Psihat(\cdot,0) \Psi(\cdot,0) = \target \\[3pt]
            \Psihat(\cdot,T) \Psi(\cdot,T) = \rho_0
        \end{array}.  \label{eq:sb2-pde-s}
    \end{align}

    Apply the It\^o formula to $v := \log \Psi(\Xbar_s,s)$, where $\Xbar_s$ evolves along the reversed SDE \eqref{eq:rsde}.
    \begin{align*}
        \rd \log \Psi
        = \fracpartial{\log \Psi}{s}\ds
        &+ \br{\nabla \log \Psi^\T (- f + \sigma^2 \nabla \log \Psihat) + \frac{\sigma^2}{2} \Delta \log \Psi } \ds
        + \sigma\nabla \log \Psi^\T \rd W_s,
    \end{align*}
    and notice that the PDE of $\fracpartial{\log \Psi}{s}$ now obeys
    \begin{align*}
        \fracpartial{\log \Psi}{s}
        &= \markblue{\frac{1}{\Psi}} \pr{
            \nabla \cdot (\Psi f) + \markblue{ \frac{\sigma^2}{2} \Delta \Psi} - F \Psi
        } \\
        &= \nabla \log \Psi^\T f + \nabla \cdot f +
        \markblue{\frac{\sigma^2}{2} \norm{\nabla \log \Psi}^2 + \frac{\sigma^2}{2} \Delta\log \Psi}
            - F.
    \end{align*}
    This yields
    \begin{equation}
    \begin{split}
        \rd \log \Psi
    &= \br{
        \nabla \cdot f
        + \frac{\sigma^2}{2} \norm{\nabla \log \Psi}^2 + \sigma^2\nabla \log \Psi^\T \nabla \log {\Psihat} + \sigma^2\Delta\log \Psi - F
        } \ds + \sigma \nabla \log \Psi^\T \rd W_s \\
        &= \br{
            \nabla \cdot (f + \sigma^2 \nabla \log \Psi)
            + \frac{\sigma^2}{2} \norm{\nabla \log \Psi}^2 + \sigma^2\nabla \log \Psi^\T \nabla \log {\Psihat} - F
        } \ds + \sigma \nabla \log \Psi^\T \rd W_s.
        \label{eq:thm1-Ys}
    \end{split}
    \end{equation}

    Similarly, apply the It\^o formula to $v := \log \Psihat(\Xbar_s,s)$, where $\Xbar_s$ follows the same reversed SDE \eqref{eq:rsde}.
    \begin{align*}
        \rd \log \Psihat
        = \fracpartial{\log \Psihat}{s}\ds
        &+ \br{\nabla \log \Psihat^\T (- f + \sigma^2 \nabla \log \Psihat) + \frac{\sigma^2}{2} \Delta \log \Psihat } \ds
        + \sigma\nabla \log \Psihat^\T \rd W_s,
    \end{align*}
    and notice that the PDE of $\fracpartial{\log \Psihat}{s}$ obeys
    \begin{align*}
        \fracpartial{\log \Psihat}{s}
        = \markgreen{\frac{1}{\Psihat}} \pr{
            \nabla \Psihat^\T f \markgreen{~- \frac{\sigma^2}{2} \Delta \Psihat} + F \Psihat
        }
        =  \nabla \log \Psihat^\T f \markgreen{~- \frac{\sigma^2}{2} \norm{\nabla \log \Psihat}^2 - \frac{\sigma^2}{2} \Delta\log \Psihat} + F.
    \end{align*}
    This yields
    \begin{align}
        \rd \log \Psihat
        &= \br{
        \frac{1}{2}\norm{\sigma \nabla \log \Psihat}^2 + F} \ds + \sigma \nabla \log \Psihat^\T \rd W_s. \label{eq:thm1-Yhats}
    \end{align}

    Finally, with a nonlinear FK transformation similar to \eqref{eq:nkc-yz},
    \begin{equation}
        \begin{alignedat}{2}
        Y_s \equiv Y(\Xbar_s, s) &= \log \Psi(\Xbar_s, s), \qquad
        Z_s \equiv Z(\Xbar_s, s) &&= \sigma~\nabla \log \Psi(\Xbar_s, s), \\
        \Yhat_s \equiv \Yhat(\Xbar_s, s) &=          \log \Psihat(\Xbar_s, s), \qquad
        \Zhat_s \equiv \Zhat(\Xbar_s, s) &&= \sigma~\nabla \log \Psihat(\Xbar_s, s),
        \end{alignedat}
    \end{equation}
    we can rewrite (\ref{eq:rsde}, \ref{eq:thm1-Ys}, \ref{eq:thm1-Yhats}) as the second FBSDEs system in \eqref{eq:fbsdes}.
    \begin{subequations}
    \begin{empheq}[box=\widefbox]{align*}
        \rd \Xbar_s &= \pr{-f_s + \sigma \Zhat_s} \ds + \sigma \rd W_s \\
        \rd Y_s &= \pr{\frac{1}{2} \norm{Z_s}^2 +
                        \nabla \cdot (\sigma Z_s +f_s) + Z_s^\T \Zhat_s {- F_s}
                    } \ds + Z_s^\T \rd W_s \\
        \rd \Yhat_s &= \pr{ \frac{1}{2} \norm{\Zhat_s}^2 {+ F_s} } \ds + \Zhat_s^\T \rd W_s
    \end{empheq} %
    \end{subequations}
    where
    \begin{align*}
        f_s := f(\Xbar_s, \exp(Y_s+\Yhat_s)), \qquad F_s := F(\Xbar_s, \exp(Y_s+\Yhat_s)).
    \end{align*}
    We conclude the proof.
\end{proof}

\subsubsection{Proof of Proposition~\ref{prop:td}} \label{sec:a3.prop-td}

\begin{proof}
We will only prove the TD objective \eqref{eq:td2-single} for the time coordinate $t$, as all derivations can be adopted similarly to its reversed coordinate $s := T-t$.

Given a realization of the parametrized SDE \eqref{eq:fsdet} w.r.t. some fixed step size $\delta t$, \ie
\begin{align*}
    X_{t+\delta t}^\theta = X_t^\theta +
    \markcc{\pr{f_t + \sigma Z_\theta(X_t^\theta,t)}\delta t}
    + \markbb{\delta W_t},
    \quad \delta W_t \sim \calN(\mathbf{0},  \delta t \mI),
\end{align*}
we can represent the trajectory compactly by a sequence of tuples $\mX_t^\theta \equiv (X_t^\theta, Z_t^\theta, \delta W_t)$ sampled on some discrete time grids, $ t \in \{0, \delta t, \cdots, T-\delta t, T\}$.
The incremental change of $\Yhat_t$, \ie the \textit{r.h.s.} of \eqref{eq:bsdet2}, can then be computed by
\begin{align*}
    \delta \Yhat_t(\mX_t^\theta) :=
    \markcc{\pr{
                \frac{1}{2} \norm{\Zhat(X_t^\theta,t)}^2 + \nabla \cdot (\sigma \Zhat(X_t^\theta,t) -f_t) + \Zhat(X_t^\theta,t)^\T Z_t^\theta {- F_t}}
            \delta t} + \markbb{\Zhat(X_t^\theta,t)^\T \delta W_t},
\end{align*}
where $\Zhat(\cdot,\cdot)$ is the (parametrized) backward policy and we denote $Z_t^\theta := Z_\theta(X_t^\theta,t)$ for simplicity.
At the equilibrium when the FBSDE system \eqref{eq:fbsdet} is satisfied,
the SDE \eqref{eq:bsdet2} must hold.
This suggests the following equality:
\begin{align}
    \Yhat(X_{t+\delta t}^\theta,{t+\delta t}) =
    \Yhat(X_t^\theta,t) + \delta \Yhat_t(\mX_t^\theta).
    \label{eq:a}
\end{align}
Hence, we can interpret the \textit{r.h.s.} of \eqref{eq:a} as the single-step TD target $\TDhat_{t+\delta t}^\text{single}$, which yields the expression in \eqref{eq:td2-single}.
The multi-step TD target can be constructed accordingly as standard practices \citep{van2016effective,hessel2018rainbow},
and either TD target can be used to construct the TD objective for the parametrized function $\Yhat_\phi \approx \Yhat$, which further yields \eqref{eq:L-td}.
\end{proof}

\subsubsection{Proof of Proposition~\ref{prop:sufficient}} \label{sec:a3.prop-conv}

\begin{proof}
    We first prove the necessity. Suppose the parametrized functions $(Y_\theta, Z_\theta, \Yhat_\phi, \Zhat_\phi)$ satisfy the SDEs in (\ref{eq:fbsdet},\ref{eq:fbsdes}), it can be readily seen that the TD objectives $\Ltd(\phi)$ and $\Ltd(\theta)$ shall both be minimized, as the parametrized functions satisfy (\ref{eq:bsdet2},\ref{eq:bsdes1}).
    Next, notice that \eqref{eq:fbsdet} implies
    \begin{align*}
        Y_T^\theta + \Yhat_T^\phi
        &= \pr{Y_0^\theta + \int_0^T \rd Y^\theta_t} + \pr{\Yhat_0^\phi + \int_0^T \rd \Yhat^\phi_t} \\
        \Rightarrow
        0
        &= \E_{q^\theta} \br{\markaa{\pr{ Y_0^\theta + \Yhat_0^\phi }} + \markcc{\int_0^T \pr{ \rd Y^\theta_t + \rd \Yhat^\phi_t }} - \markbb{\pr{ Y_T^\theta + \Yhat_T^\phi }}} \\
        \overset{(*)}&{=} \markaa{\E_{q_0^\theta}\br{ \log \rho_0 }} + \markcc{\int_0^T \E_{q_t^\theta} \br{
                \frac{1}{2}\norm{Z_t^\theta + \Zhat_t^\phi}^2 + \nabla \cdot \pr{ \sigma \Zhat_t^\phi - f_t }
            } \dt} - \markbb{\E_{q_T^\theta}\br{ \log \target }} \\
        \overset{(**)}&{=}
            \E_{q_0^\theta}\br{ \log \rho_0 } + \markcc{\KL(q^\theta || q^\phi) - \E_{q^\theta} \br{ \log \frac{\rho_0}{\target}}} - \E_{q_T^\theta}\br{ \log \target } \\
        &= \KL(q^\theta || q^\phi),
    \end{align*}
    where (*) is due to (\ref{eq:bsdet1},\ref{eq:bsdet2}) and (**) invokes Proposition~\ref{prop:a}.
    The fact that $\Lipf(\phi) \propto \KL(q^\theta || q^\phi) = 0$ (recall Lemma~\ref{lemma:kl-sbfbsde}) suggests that the objective $\Lipf(\phi)$ is minimized when \eqref{eq:fbsdet} holds.
    Finally, as similar arguments can be adopted to $\Lipf(\theta) \propto \KL(q^\phi || q^\theta) = 0 $ when \eqref{eq:fbsdes} holds, we conclude that all losses are minimized when the parameterized functions satisfy the FBSDE systems (\ref{eq:fbsdet},\ref{eq:fbsdes}).

    We proceed to proving the sufficiency, which is more involved. First, notice that
    \begin{align}
        \Lipf(\phi) \text{ is minimized} \Leftrightarrow \KL(q^\theta || q^\phi) = 0 \Leftrightarrow
        \forall s\in[0,T], \text{ } Z_s^\theta + \Zhat_s^\phi - \sigma {\nabla} \log q^\theta_s = 0, \label{eqq:a} \\
        \Lipf(\theta) \text{ is minimized} \Leftrightarrow \KL(q^\phi || q^\theta) = 0 \Leftrightarrow
        \forall t\in[0,T], \text{ } Z_t^\theta + \Zhat_t^\phi - \sigma {\nabla} \log q^\phi_t = 0, \label{eqq:b}
    \end{align}
    as implied by \eqref{eq:L-ipf22}.
    If $\Ltd(\phi)$ and $\Ltd(\theta)$ are minimized, the following relations must also hold
    \begin{align}
        \rd \Yhat^\phi_t &= \pr{\frac{1}{2} \norm{\Zhat^\phi_t}^2 +
        \nabla \cdot (\sigma \Zhat^\phi_t -f_t) + {Z_t^\theta}^\T \Zhat^\phi_t  - F_t
        } \dt + {\Zhat^\phi_t}{}^\T \rd W_t, \label{eqq:c} \\
        \rd Y^\theta_s &= \pr{\frac{1}{2} \norm{Z^\theta_s}^2 +
        \nabla \cdot (\sigma Z^\theta_s + f_s ) + {Z_s^\theta}^\T \Zhat^\phi_s  - F_s
        } \dt + {Z^\theta_s}{}^\T \rd W_s. \label{eqq:d}
    \end{align}

    Now, notice that the Fokker Plank equation of the parametrized forward SDE \eqref{eqq:j} obeys
    \begin{align*}
        \fracpartial{q_t^\theta}{t}
        = - \nabla \cdot \pr{ q_t^\theta \pr{ f_t + \sigma Z^\theta_t } } + \frac{1}{2} \sigma^2 \Delta q_t^\theta,
    \end{align*}
    which implies that (\textit{c.f.} Lemma \ref{lemma:sbp}),
    \begin{align}
        \fracpartial{\log q_t^\theta}{t}
        = - \nabla \cdot \pr{ f_t + \sigma Z^\theta_t } - {\nabla \log q_t^\theta}^\T \pr{ f_t + \sigma Z^\theta_t } + \frac{\sigma^2}{2} \pr{ \Delta \log q_t^\theta + \norm{\nabla \log q_t^\theta}^2 }.
        \label{eqq:42}
    \end{align}
    Invoking Ito lemma yields:
    \begin{align*}
        \rd \log q_t^\theta &=
        \fracpartial{\log q_t^\theta}{t} \dt + \br{
            {\nabla \log q_t^\theta}^\T \pr{ f_t + \sigma Z^\theta_t} + \frac{\sigma^2}{2} \Delta \log q_t^\theta
        } \dt + \sigma {\nabla \log q_t^\theta}^\T \rd W_t \\
        \overset{\text{\eqref{eqq:42}}}&{=}
        \br{
            - \nabla \cdot \pr{ f_t + \sigma Z^\theta_t } + \sigma^2 \Delta \log q_t^\theta
            + \frac{\sigma^2}{2} \norm{\nabla \log q_t^\theta}^2
        } \dt + \sigma {\nabla \log q_t^\theta}^\T \rd W_t \\
        &=
        \br{
            - \nabla \cdot \pr{ f_t + \sigma Z^\theta_t - \markgreen{\sigma}^2 \markgreen{\nabla  \log q_t^\theta} }
            + \frac{1}{2} \norm{\markgreen{\sigma \nabla \log q_t^\theta}}^2
        } \dt + \markgreen{\sigma {\nabla \log q_t^\theta}}^\T \rd W_t \\
        \overset{\markgreen{(*)}}&{=}~
        \br{
            - \nabla \cdot \pr{ f_t + \sigma Z^\theta_t - \sigma \pr{ \markgreen{Z^\theta_t + \Zhat^\phi_t}  } }
            + \frac{1}{2} \norm{\markgreen{Z^\theta_t + \Zhat^\phi_t}}^2
        } \dt + \markgreen{\pr{Z^\theta_t + \Zhat^\phi_t}}^\T \rd W_t \\
        &=
        \br{
            \nabla \cdot \pr{  \sigma \Zhat^\phi_t - f_t }
            + \frac{1}{2} \norm{Z^\theta_t + \Zhat^\phi_t}^2
        } \dt + \pr{Z^\theta_t + \Zhat^\phi_t}^\T \rd W_t, \numberthis \label{eqq:e}
    \end{align*}
    where \markgreen{(*)} is due to \eqref{eqq:a}.
    Subtracting \eqref{eqq:c} from \eqref{eqq:e} yields
    \begin{align}
        \rd \log q_t^\theta - \rd \Yhat^\phi_t &=
        \pr{\frac{1}{2} \norm{Z_t^\theta}^2 + F_t
                        } \dt + {Z_t^\theta}^\T \rd W_t. \label{eqq:f}
    \end{align}
    Now, using the fact that
    $Z_\theta := \sigma \nabla Y_\theta$ and  $\Zhat_\phi := \sigma \nabla \Yhat_\phi$, we know that
    \begin{align*}
        Z_t^\theta + \Zhat_t^\phi - \sigma \nabla \log q^\theta_t = 0 \Rightarrow
        Y_t^\theta + \Yhat_t^\phi = \log q^\theta_t + c_t,
    \end{align*}
    for some function $c_t \equiv c(t)$.
    Hence, \eqref{eqq:f} becomes
    \begin{align}
        \rd Y^\theta_t - \rd c_t &=
        \pr{\frac{1}{2} \norm{Z_t^\theta}^2 + F_t
                        } \dt + {Z_t^\theta}^\T \rd W_t. \label{eqq:gf}
    \end{align}
    Now we prove that $\forall t \in (0,T), \rd c_t = 0$ by contradiction.
    First, notice that $c_t$ can be derived analytically as
    \begin{align*}
        c_t
        &= Y_t^\theta + \Yhat_t^\phi - \log q^\theta_t \\
        &= \int_0^t \pr{
            \markblue{\rd Y_\tau^\theta} \markgreen{+~ \rd \Yhat_\tau^\phi - \rd \log q_\tau^\theta}
        } \\
        \overset{{(*)}}&{=}~ \int_0^t \pr{
            \markblue{\pr{
                \fracpartial{Y_\tau^\theta}{\tau} + {\nabla Y_\tau^\theta}^\T \pr{f_\tau + \sigma Z^\theta_\tau } + \frac{\sigma^2}{2} \Delta Y_\tau^\theta
            }} \markgreen{ - \pr{\frac{1}{2} \norm{Z_\tau^\theta}^2 + F_\tau}}
        } \rd \tau \\
        &\qquad
        + \int_0^t \pr{
            \markblue{\sigma {\nabla Y_\tau^\theta}^\T \rd W_\tau} - \markgreen{{Z_\tau^\theta}^\T \rd W_\tau}
        }
        \\
        \overset{{(**)}}&{=}~ \int_0^t \pr{
            \fracpartial{Y_\tau^\theta}{t} - \pr{
                - {\nabla Y_\tau^\theta}^\T f_\tau - \frac{1}{2} \norm{\sigma \nabla Y_\tau^\theta}^2
                - \frac{\sigma^2}{2} \Delta Y_\tau^\theta
                + F_\tau
            }
        } \rd \tau, \numberthis \label{eqq:g}
    \end{align*}
    where (*) invokes \markblue{the following Ito lemma} and \markgreen{substitutes \eqref{eqq:f}},  %
    \begin{align*}\markblue{
        \rd Y_\tau^\theta
        = \fracpartial{Y_\tau^\theta}{\tau} \rd \tau  + \br{
            {\nabla Y_\tau^\theta}^\T \pr{f_\tau + \sigma Z^\theta_\tau } + \frac{\sigma^2}{2} \Delta Y_\tau^\theta
        } \dt  + \sigma {\nabla Y_\tau^\theta}^\T \rd W_\tau,}
    \end{align*}
    and (**) substitutes the definition $Z_\tau^\theta := \sigma \nabla Y_\tau^\theta$.
    Equation \eqref{eqq:g} has an intriguing implication, as one can verify that its integrand is the \emph{residual of the parametrized HJB $Y_\theta = - u_\theta \approx -u$} (recall \eqref{eq:mfg2-pde} and \eqref{eq:hf}).
    It is straightforward to see that,
    the residual shall also be preserved after the parametrized HJB is expanded by Ito lemma w.r.t. the backward parametrized SDE \eqref{eqq:k}.
    That is, the following equation similar to \eqref{eqq:gf} must hold for the function $c_s := c(T-t)$:
    \begin{align*}
        \rd Y^\theta_s + \rd c_s &= \pr{\frac{1}{2} \norm{Z^\theta_s}^2 +
            \nabla \cdot (f_s + \sigma Z^\theta_s ) + {Z_s^\theta}^\T \Zhat^\phi_s  - F_s
            } \dt + {Z^\theta_s}{}^\T \rd W_s,
    \end{align*}
    which contradicts \eqref{eqq:d}. Hence, we must have $\rd c_s = \rd c_t = 0$,
    and \eqref{eqq:gf} becomes
    \begin{align}
        \rd Y^\theta_t &=
        \pr{\frac{1}{2} \norm{Z_t^\theta}^2 + F_t
                        } \dt + {Z_t^\theta}^\T \rd W_t. \label{eqq:h}
    \end{align}
    In short, we have shown that, for the parametrized forward \eqref{eqq:j} and backward \eqref{eqq:k} SDEs, the fact that (\ref{eqq:c}, \ref{eqq:d}) hold implies that \eqref{eqq:h} holds, providing $\Lipf$ is minimized. The exact same statement can be repeated to prove that
    \begin{align}
        \rd \Yhat^\phi_s &=
        \pr{\frac{1}{2} \norm{\Zhat_s^\phi}^2 + F_s
                        } \ds + {\Zhat_s^\phi}{}^\T \rd W_s. \label{eqq:i}
    \end{align}
    Therefore, if the combined objectives are minimized, \ie (\ref{eqq:a}, \ref{eqq:b}, \ref{eqq:c}, \ref{eqq:d}) hold, the parametrized functions $(Y_\theta, Z_\theta, \Yhat_\phi, \Zhat_\phi)$ satisfy (\ref{eqq:25}, \ref{eqq:c}, \ref{eqq:h}, \ref{eqq:d}, \ref{eqq:i}), \ie they satisfy the FBSDE systems (\ref{eq:fbsdet},\ref{eq:fbsdes}) in Theorem \ref{thm:mfg-fbsde}.
\end{proof}

\subsection{Additional Derivations \& Remarks in Sec. \ref{sec:3} and \ref{sec:4}} \label{sec:a4}

\subsubsection{Hopf-Cole transform} \label{sec:a4.1}
    Recall the Hopf-Cole transform
    \begin{align*}
        \Psi(x,t) := \exp\pr{-u(x,t)}, \quad \Psihat(x,t) := \rho(x,t)\exp\pr{u(x,t)}.
    \end{align*}
    Standard ordinary calculus yields
    \begin{align}
        \nabla\Psi
        &= -\expu\Du, \quad
        &&\Delta\Psi
        = \expu \left[
                \norm{\Du}^2 - \Delta u
            \right], \label{eq:Dpsi} \\
        \nabla\Psihat
        &= \expv\left( \rho \Du +\Dm \right),
        &&\Delta\Psihat
        = \expv \br{
            \rho\norm{\Du}^2 + 2\Dm^\T\Du + \Delta \rho + \rho \Delta u
        }. \label{eq:Dhpsi}
    \end{align}
    Hence, we have
    \begin{align*}
        \fracpartial{\Psi}{t}
            &= \expu\left(-\fracpartial{u}{t}\right) \\
            \overset{\text{\eqref{eq:mfg2-pde}}}&{=}
            \expu\pr{- \frac{1}{2}\norm{\sigma\Du}^2 + \Du^\T f + \frac{1}{2}\sigma^2 \Delta u + F} \\
            \overset{\text{\eqref{eq:Dpsi}}}&{=}
            - \frac{1}{2}\sigma^2 \Delta \Psi - \nabla\Psi^\T f + F\Psi, \numberthis \\
        \fracpartial{\Psihat}{t}
        &= \expv\pr{ \fracpartial{\rho}{t} + \rho \fracpartial{u}{t} } \\
        \overset{\text{\eqref{eq:mfg2-pde}}}&{=}
            \expv\pr{ \pr{\nabla\cdot(\rho(\sigma^2\Du - f)) + \frac{1}{2}\sigma^2 \Delta \rho} + \rho \pr{\frac{1}{2}\norm{\sigma\Du}^2 - \Du^\T f - \frac{1}{2}\sigma^2 \Delta u - F} } \\
        &= \expv \pr{
            \markgreen{
                \sigma^2 \pr{ \rho \Delta u + \Dm^\T\Du + \frac{1}{2}\Delta \rho + \frac{\rho}{2}\norm{\Du}^2 - \frac{\rho}{2} \Delta u }
            }
            \markblue{-~\nabla\rho^\T f} - \rho \nabla \cdot f \markblue{~-\rho\nabla u^\T f}
            - \rho F
        } \\
        \overset{\text{\eqref{eq:Dhpsi}}}&{=}
            \markgreen{\frac{1}{2}\sigma^2 \Delta \Psihat} \markblue{~-\nabla \Psihat^\T f} - \Psihat \nabla \cdot f - \Psihat F,  \numberthis
    \end{align*}
    which yields \eqref{eq:sb2-pde} by noticing that
    $\nabla \cdot (\Psihat f) = \nabla \Psihat^\T f + \Psihat \nabla \cdot f$.

\subsubsection{Remarks on convergence} \label{sec:a4.conv}

    The alternating optimization scheme proposed in Alg. \eqref{alg:jacobi} can be compactly presented as $\min_\phi \KL(q^\theta || q^\phi) + \E_{q^\theta} [\Ltd(\phi)]$ and $\min_\theta \KL(q^\phi | q^\theta) + \E_{q^\phi} [\Ltd(\theta)]$.
    Despite that the procedure seems to resemble IPF, which optimizes between $\min_\phi \KL(q^\phi | q^\theta)$ and $\min_\theta \KL(q^\theta | q^\phi)$, we stress that they differ from each other in that the the KLs are constructed with different directions.

    In cases where the TD objectives are discarded, prior work \citep{de2021diffusion} has proven that minimizing the forward KLs admit similar convergence to standard IPF (which minimizes the reversed KLs). This is essentially the key to developing scalable methods, since the parameter being optimized (\eg $\theta$ in $\KL(q^\phi|q^\theta)$) in forward KLs differs from the parameter used to sample expectation (\eg $\E_{q^\phi}$). Therefore, the computational graph of the SDEs can be dropped, yielding a computationally much efficient framework. These advantages have been adopted in \citep{de2021diffusion,chen2021likelihood} and also this work for solving higher-dimensional problems.

    However, when we need TD objectives to enforce the MF structure, as appeared in all the MFGs in this work, the combined objective does not correspond to IPF straightforwardly. Despite that the alternating procedure in Alg. \eqref{alg:jacobi} is mainly inspired by prior SB methods \citep{de2021diffusion,chen2021likelihood}, the training process of DeepGSB is perhaps closer to TRPO \citep{schulman2015trust}, which iteratively updates the policy using the off-policy samples generated from the previous stage: $\pi^{(i+1)} = \argmin_\pi  \KL(\pi^{(i)} || \pi) + \E_{\pi^{(i)}} [\mathcal{L}(\pi)]$. TRPO is proven to enjoy monotonic improvement over iterations (\ie local convergence).

\subsubsection{Functional derivative of MF potential functions} \label{sec:a4.2}

\def\dx{{ \rd x }}
\def\dy{{ \rd y }}
\def\calFent{{ \calF_\text{entropy} }}
\def\Fent{{ F_\text{entropy} }}
\def\calFcon{{ \calF_\text{congestion} }}
\def\Fcon{{ F_\text{congestion} }}

Given a functional $\calF:  \calP(\sR^d) \to \sR$ on the space of probability measures,
its functional derivative $F(x,\rho) := \frac{\delta\calF(\rho)}{\delta\rho}(x)$
satisfies the following equation
\begin{align*}
    \lim_{h \to 0} \frac{\calF(\rho + h w) - \calF(\rho)}{h}
    = \int_{\sR^d} F(x, \rho) w(x) \rd x
\end{align*}
for any function $w \in L^2(\sR^d)$.
Hence, the derivative of the entropy MF functional $\calFent := \int_{\sR^d} \rho(x) \log\rho(x) \dx$
can be derived as
\begin{align*}
    &\lim_{h \to 0} \frac{1}{h} \Big( \calFent(\rho + h w) - \calFent(\rho) \Big) \\
    =& \lim_{h \to 0} \frac{1}{h} \Big( \int_{\sR^d} \Big( h w(x) \log\rho(x) + \rho(x) \frac{h w(x)}{\rho(x)} + \calO(h^2) \Big) \dx \Big) \\
    =& \int_{\sR^d} \Big( w(x) \log\rho(x) + w(x) \Big) \dx
    = \int_{\sR^d} \underbrace{\Big( \log\rho(x) + 1 \Big)}_{:= \Fent(x,\rho)}  w(x) \dx. \numberthis
\end{align*}
Similarly, consider the congestion MF functional $\calFcon := \int_{\sR^d} \int_{\sR^d} \frac{1}{\norm{x-y}^2+1} \rho(x) \rho(y) \dx \dy$.
Its derivation can be computed by
\begin{align*}
    &\lim_{h \to 0} \frac{1}{h} \Big( \calFcon(\rho + h w) - \calFcon(\rho) \Big) \\
    = &\lim_{h \to 0} \frac{1}{h} \Big( \int_{\sR^d} \int_{\sR^d} \frac{1}{\norm{x-y}^2+1} \Big( \rho(x) h w(y) + h w(x) \rho(y) + \calO(h^2) \Big) \dx\dy \Big) \\
    =& \int_{\sR^d} \int_{\sR^d} \frac{1}{\norm{x-y}^2+1} \Big( \rho(x) w(y) + w(x) \rho(y) \Big) \dx\dy \\
    =& \int_{\sR^d} \underbrace{\int_{\sR^d} \frac{2}{\norm{x-y}^2+1} \dy }_{:= \Fcon(x,\rho)}  w(x) \dx. \numberthis
\end{align*}
We hence conclude the expressions of $\Fent$ and $\Fcon$ in \eqref{eq:F}.

\subsection{Experiment Details} \label{sec:a5}

\subsubsection{Setup} \label{sec:a5.1}

\paragraph*{Hyperparameters}

Table~\ref{table:2-expand} summarizes the hyperparameters in each MFG, including
the dimension $d$ of the state space,
the diffusion scalar $\sigma$,
the time horizon $T$,
the discretized time step $\delta t$ (and $\delta s$),
the MF base drift $f(x,\rho)$,
the MF interaction $F(x,\rho)$,
and the mean/covariance of the boundary distributions $\rho_0$ and $\target$ (
    note that all MFGs adopt Gaussians as their boundary distributions
).
Note that in the 1000-dimensional opinion MFG,
we multiply the polarized dynamic $\bar{f}_\text{polarize}$ by $6$
to ensure that the high-dimensional dynamics yield polarization within the time horizon.
Meanwhile, a smaller step size $\delta t = 0.006$ is adopted
so that the discretization error from the relatively large drift is mitigated.
As mentioned in Sec.~\ref{sec:4}, we adopt zero and constant base drift $f$ respectively for GMM and V-neck/S-tunnel.

\begin{table}[t]
    \centering
    \captionsetup{type=table}
    \captionsetup{justification=centering}
    \caption{
        Hyperparameters in each MFG.
        Note that $\mathbf{0} \in \sR^d$ denotes zero vector, $\mI \in \sR^{d\times d}$ \\ denotes identity matrix, and $\mathrm{diag}(\vv)\in \sR^{d\times d}$, where $\vv \in \sR^d$, denotes diagonal matrix.
    }
    \label{table:2-expand}
    \vskip 0.05in
    \centering
    \begin{tabular}{rccccc}
        \toprule
                             & \GMM & \Vneck & \Stunnel & \multicolumn{2}{c}{Opinion} \\
        \midrule
        $d$                  & $2$  & $2$    & $2$      & $2$   & $1000$ \\[2pt]
        $\sigma$             & $1$  & $1$    & $1$      & $0.1$ & $0.5$  \\[2pt]
        $T$                  & $1$  & $2$    & $3$      & $3$   & $3$    \\[2pt]
        $\delta t$           & $0.01$ & $0.01$ & $0.01$ & $0.01$ & $0.006$ \\[2pt]
        $f(x,\rho)$          & $[0, 0]^\T$ & $[6, 0]^\T$ & $[6, 0]^\T$ & $\bar{f}_\text{polarize}$ & $6\cdot\bar{f}_\text{polarize}$ \\[2pt]
        {Diffusion steps}      & {$100$}  & {$200$}    & {$300$}      & {$300$}   & {$500$}   \\[2pt]
        {$K$\protect\footnotemark}                  & {$250$}  & {$250$}    & {$500$}      & {$100$}   & {$250$}   \\[2pt]
        {Alternating stages\protect\footnotemark}       & {$40$}   & {$40$}     & {$30$}       & {$40$}    & {$90$}    \\[2pt]
        {Total training steps} & {$20$k}  & {$20$k}    & {$30$k}      & {$8$k}    & {$45$k}   \\[2pt]
        {Mean of $\rho_0$} & {$\mathbf{0}$}  & {$\bvec{-7\\0}$}    & {$\bvec{-11\\-1}$}      & {$\mathbf{0}$}    & {$\mathbf{0}$}   \\[10pt]
        {Mean of $\target$} & {\specialcell[c]{$e^{16\cdot(\frac{\pi}{4})i},$ \\ $i \in \{0,\cdots,7\}$}}  & {$\bvec{7\\0}$}    & {$\bvec{11\\1}$}      & {$\mathbf{0}$}    & {$\mathbf{0}$}   \\[4pt]
        {Covariance of $\rho_0$} & {$\mI$}  & {$0.2\mI$}    & {$0.5\mI$}      & {$\mathrm{diag}(\bvec{0.5\\0.25})$}    & {$\mathrm{diag}(\bvec{4\\0.25\\\vdots\\0.25})$}   \\[4pt]
        {Covariance of $\target$} & {$\mI$}  & {$0.2\mI$}    & {$0.5\mI$}      & {$3\mI$}    & {$3\mI$}   \\[2pt]
        \bottomrule
    \end{tabular}
    \vskip -0.1in
\end{table}
\footnotetext[7]{
    We note that, unlike SB-FBSDE \citep{chen2021likelihood}, the number of training iterations at each stage (\ie the $K$ in Alg. \ref{alg:jacobi}) is kept \emph{fixed} throughout training.
}
\footnotetext{
  Here, we refer \emph{one alternating stage} to a complete cycling through $2K$ training iterations in Alg. \ref{alg:jacobi}.
}

\paragraph*{Training}

All experiments are conducted on 3 TITAN RTXs and 1 TITAN V100, where the V100 is located on the Amazon Web Service (AWS).
We use the multi-step TD targets in \eqref{eq:td-multi} for all experiments
and adopt huber norm for the TD loss in \eqref{eq:L-td}.
As for the FK consistency loss $\Lfk$,
we use $\ell1$ norm for GMM and opinion MFGs, and huber norm for the rest.

\paragraph*{Network architecture}

All networks $(Y_\theta, Z_\theta, \Yhat_\phi, \Zhat_\phi)$ take $(x,t)$ as inputs and follow
\begin{align*}
    \texttt{out} = \texttt{out\_mod(}\texttt{x\_mod(}~x~\texttt{)} + \texttt{t\_mod(timestep\_embedding(}~t~\texttt{)))},
\end{align*}
where \texttt{timestep\_embedding($\cdot$)} is the standard sinusoidal embedding.

For crowd navigation MFGs, these modules consist of
2 to 4 fully-connected layers (\texttt{Linear}) followed by the Sigmoid Linear Unit (\texttt{SiLU}) activation functions \citep{elfwing2018sigmoid}, \ie
\begin{align*}
    \texttt{t\_mod} &= \texttt{Linear} \to \texttt{SiLU} \to \texttt{Linear} \\
    \texttt{x\_mod} &= \texttt{Linear} \to \texttt{SiLU} \to \texttt{Linear} \to \texttt{SiLU} \to \texttt{Linear} \to \texttt{SiLU} \to \texttt{Linear} \\
    \texttt{out\_mod} &= \texttt{Linear} \to \texttt{SiLU} \to \texttt{Linear} \to \texttt{SiLU} \to \texttt{Linear}
\end{align*}
As for 1000-dimensional opinion MFG, we keep the same \texttt{t\_mod} and \texttt{out\_mod} but
adopt residual networks with 5 residual blocks for \texttt{x\_mod}.
For DeepGSB-ac, we set the hidden dimension of \texttt{Linear} to 256 and 128 respectively
for the policy networks $(Z_\theta,\Zhat_\phi)$ and the critic networks $(Y_\theta,\Yhat_\phi)$,
whereas for DeepGSB-c, we set the hidden dimension of \texttt{Linear} to 200 for the critic networks $(Y_\theta,\Yhat_\phi)$.

\paragraph*{Implementation of prior methods \citep{ruthotto2020machine, lin2021alternating, chen2021density}}

All of our experiments are implemented with PyTorch \citep{paszke2017automatic}.
Hence, we re-implement the method in \citet{ruthotto2020machine}
by migrating their Julia codebase\footnote{
    \url{https://github.com/EmoryMLIP/MFGnet.jl}. The repository is licensed under MIT License.
} to PyTorch.
As for \citet{lin2021alternating},
their official PyTorch implementation is publicly available.\footnote{
    \url{https://github.com/atlin23/apac-net}. The repository does not specify licenses.
}
Finally, we implement \citet{chen2021density} by ourselves.
Since prior methods \citep{ruthotto2020machine, lin2021alternating, chen2021density}
were developed for a smaller class of MFGs compared to our DeepGSB (recall Table~\ref{table:1}),
we need to relax the setup of the MFG
in order for them to yield reasonable results in Fig.~\ref{fig:crowd-nav} and \ref{fig:crowd-nav2}.
Specifically, we soften the obstacle costs,
so that \citep{ruthotto2020machine, lin2021alternating}
can differentiate them properly,
and keep the same KL penalty at $ u(x,T) \approx \KL(\rho(x,T)||\target(x))$
as adopted in \citep{ruthotto2020machine, lin2021alternating}.
We stress that neither of the methods \citep{ruthotto2020machine, lin2021alternating} works well with the discontinuous $F_\text{obstacle}$ in \eqref{eq:F}.
Finally,
we discretize the 2-dimensional state space of GMM into a $40\times 40$ grid with $50$ time steps for \citep{chen2021density}.
We note that the complexity of \citep{chen2021density} scales as $\calO(\tilde{T}D^2)$,
where $\tilde{T}$ and $D$ are respectively the number of time and spatial grids,
\ie $\tilde{T} = 50$ and $D = 1600$.

\paragraph*{Evaluation}
We approximate the Wasserstein distance with the Sinkhorn divergence using the \texttt{geomloss} package.\footnote{
    \url{https://github.com/jeanfeydy/geomloss}. The repository is licensed under MIT License.
} The Sinkhorn divergence interpolates
between Wasserstein (\texttt{blur} = 0) and kernel (\texttt{blur} = $\infty$) distance given the hyperparameter \texttt{blur}.
We set \texttt{blur} = 0.05 in Table~\ref{table:3}.

\subsubsection{Additional experiments} \label{sec:a5.2}

\begin{figure}[H]
    \vskip -0.1in
    \centering
    \subfloat{
        \includegraphics[width=0.9\textwidth]{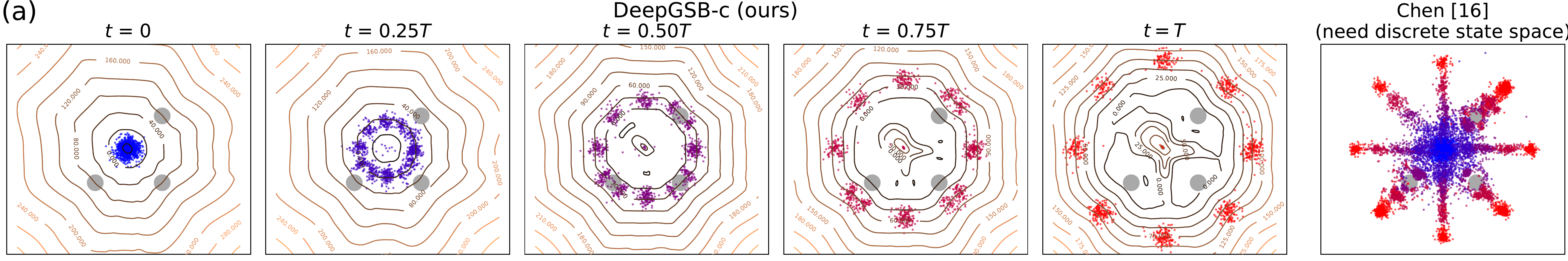}
        \label{fig:gmm2}
    }\\
    \vspace{-7.5pt}
    \subfloat{
        \includegraphics[width=0.9\textwidth]{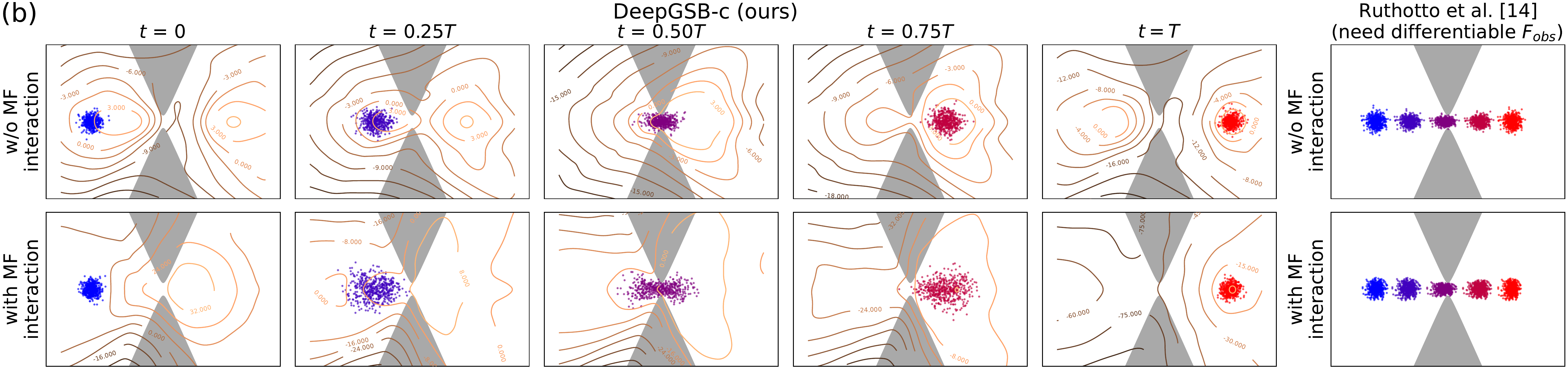}
        \label{fig:bottle2}
    }\\
    \vspace{-7.5pt}
    \subfloat{
        \includegraphics[width=0.9\textwidth]{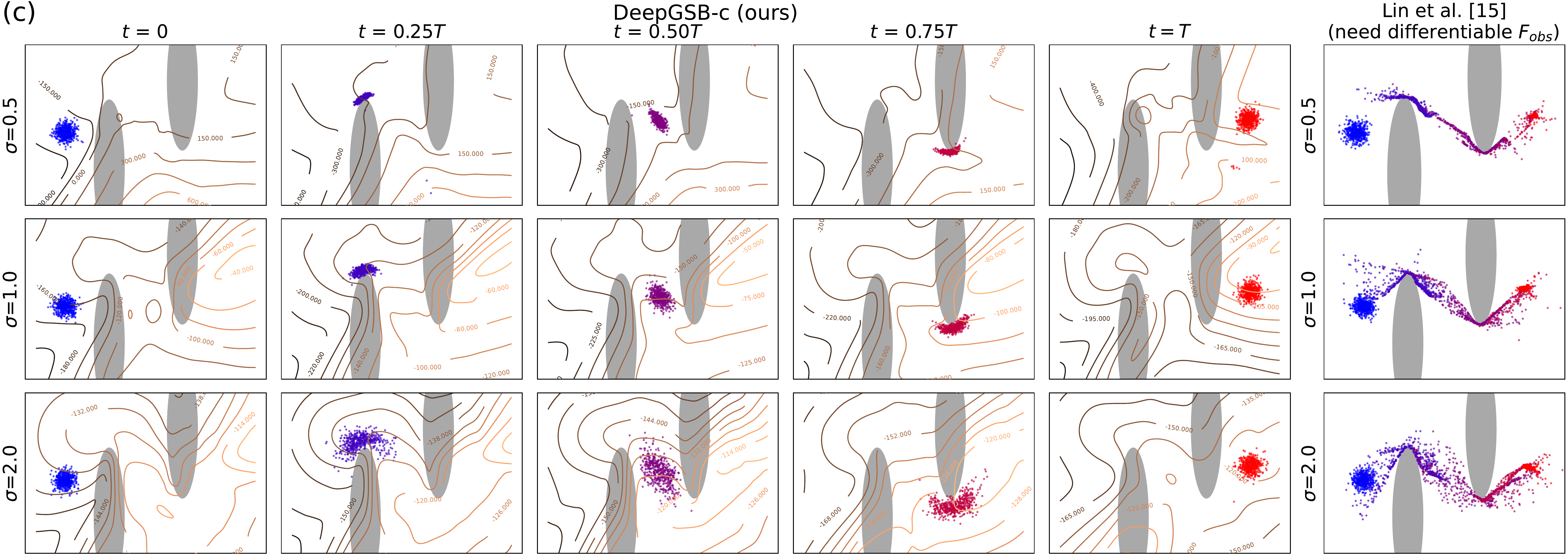}
        \label{fig:elp2}
    }
    \caption{
        Same setup as in Fig.~\ref{fig:crowd-nav} except for \textbf{DeepGSB-c}.
        This figure is best viewed in color.
    }
    \label{fig:crowd-nav2}
    \vskip -0.1in
\end{figure}

Figures \ref{fig:crowd-nav2} and \ref{fig:opinion2} reports the results for \textbf{DeepGSB-c}.
On crowd navigation MFGs,
the population snapshots guided by DeepGSB-c are visually indistinguishable from DeepGSB-ac (see Fig.~\ref{fig:crowd-nav2}~\textit{vs.}~\ref{fig:crowd-nav})
despite the visual difference in their contours.
As for 1000-dimensional opinion MFG, both DeepGSB-c and DeepGSB-ac are able to guild the population opinions toward desired
$\target$ without the entropy interaction $F_\text{entropy}$.
Figure~\ref{fig:opinion2} reports the results of DeepGSB-c in such cases.
We note, however, that when $F_\text{entropy}$ is enabled, DeepGSB-ac typically performs better than DeepGSB-c
in terms of convergence to $\target$ and training stability.

\begin{figure}[H]
    \centering
    \subfloat{
        \includegraphics[height=3.3cm]{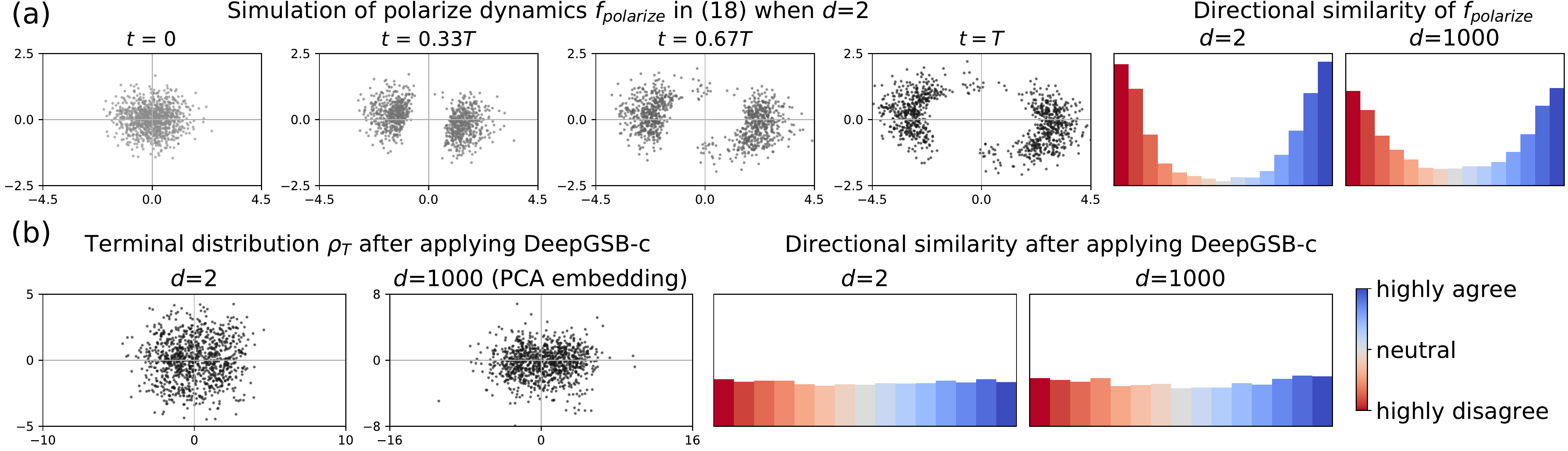}
        \label{fig:opinion2-a}%
    }
    \subfloat{
        \textcolor{white}{\rule{1pt}{1pt}}
        \label{fig:opinion2-b}%
    }
    \vskip -0.05in
    \caption{
        (a) Visualization of polarized dynamics $\bar{f}_\text{polarize}$ in 2- and 1000-dimensional opinion space,
        where the \textit{directional similarity} \citep{schweighofer2020agent} counts the histogram of cosine angle between pairwise opinions
        at the terminal distribution $\rho_T$.
        (b) \textbf{DeepGSB-c} guides $\rho_T$ to approach moderated distributions,
        hence \textit{depolarizes} the opinion dynamics.
        Note that we adopt $F:=0$ for DeepGSB-c.
        We use the first two principal components to visualize $d$=1000.
    }
    \label{fig:opinion2}
\end{figure}

\end{document}